\documentclass[twoside]{article}

\usepackage[accepted]{aistats2019}

\usepackage{extra_style}

\usepackage[hidelinks,
            colorlinks=false,
        ]{hyperref}

\begin{document}

\twocolumn[

\aistatstitle{Reparameterizing Distributions on Lie Groups}

\aistatsauthor{Luca Falorsi\textsuperscript{1,2} \And Pim de Haan\textsuperscript{1,3} \And  Tim R. Davidson\textsuperscript{1,2} \And Patrick Forr\'e\textsuperscript{1}}
\aistatsaddress{\textsuperscript{1}University of Amsterdam \;\;\; \textsuperscript{2}Aiconic \;\;\; \textsuperscript{3}UC Berkeley}
]

\begin{abstract}
    Reparameterizable densities are an important way to learn probability distributions in a deep learning setting. 
    For many distributions it is possible to create low-variance gradient estimators by utilizing a `reparameterization trick'. Due to the absence of a general reparameterization trick, much research has recently been devoted to extend the number of reparameterizable distributional families. Unfortunately, this research has primarily focused on distributions defined in Euclidean space, ruling out the usage of one of the most influential class of spaces with non-trivial topologies: \emph{Lie groups}. In this work we define a general framework to create reparameterizable densities on arbitrary Lie groups, and provide a detailed practitioners guide to further the ease of usage.
    We demonstrate how to create complex and multimodal distributions on the well known oriented group of 3D rotations, $\SO{3}$, using normalizing flows. Our experiments on applying such distributions in a Bayesian setting for pose estimation on objects with discrete and continuous symmetries, showcase their necessity in achieving realistic uncertainty estimates.
\end{abstract}

\section{INTRODUCTION} \label{sec:intro}

Formulating observed data points as the outcomes of probabilistic processes, has proven to provide a useful framework to design successful machine learning models. Thus far, the research community has drawn almost exclusively from results in probability theory limited to Euclidean and discrete space. Yet, expanding the set of possible spaces under consideration to those with a non-trivial topology has had a longstanding tradition and giant impact on such fields as physics, mathematics, and various engineering disciplines. One significant class describing numerous spaces of fundamental interest is that of \emph{Lie} groups, which are groups of symmetry transformations that are simultaneously differentiable manifolds. Lie groups include rotations, translations, scaling, and other geometric transformations, which play an important role in several application domains. Lie group elements are for example utilized to describe the rigid body rotations and movements central in robotics, and form a key ingredient in the formulation of the Standard Model of particle physics. They also provide the building blocks underlying ideas in a plethora of mathematical branches such as \emph{holonomy} in Riemannian geometry, \emph{root systems} in Combinatorics, and the \emph{Langlands Program} connecting geometry and number theory. 
    
Many of the most notable recent results in machine learning can be attributed to researchers' ability to combine probability theoretical concepts with the power of deep learning architectures, e.g. by devising optimization strategies able to directly optimize the parameters of probability distributions from samples through backpropagation. Perhaps the most successful instantiation of this combination, has come in the framework of \emph{Variational Inference} (VI) \citep{ jordan1999introduction}, a Bayesian method used to approximate intractable probability densities through optimization. Crucial for VI is the ability to posit a flexible family of densities, and a way to find the member closest to the true posterior by optimizing the parameters. 
    
These variational parameters are typically optimized using the \emph{evidence lower bound} (ELBO), a lower bound on the data likelihood. The two main approaches to obtaining estimates of the gradients of the ELBO are the \emph{score function} \citep{ paisley2012variational, mnih2014neural} also known as REINFORCE \citep{williams1992simple}, and the \emph{reparameterization trick} \citep{price1958useful, bonnet1964transformations, salimans2013fixed, kingma-vae13, pmlr-v32-rezende14}. While various works have shown the latter to provide lower variance estimates, its use is limited by the absence of a general formulation for all variational families. Although in recent years much work has been done in extending this class of reparameterizable families \citep{generalized_reparam_ruiz, accept_reject_reparam, implicit_reparam_figurnov}, none of these methods explicitly investigate the case of distributions defined on non-trivial manifolds such as Lie groups.
    
The principal contribution of this paper is therefore to extend the reparameterization trick to Lie groups. We achieve this by providing a general framework to define reparameterizable densities on Lie groups, under which the well-known Gaussian case of \citet{kingma-vae13} is recovered as a special instantiation. This is done by pushing samples from the Lie algebra into the Lie group using the exponential map, and by observing that the corresponding density change can be analytically computed. We formally describe our approach using results from differential geometry and measure theory.
    
In the remainder of this work we first cover some preliminary concepts on Lie groups and the reparameterization trick. We then proceed to present the general idea underlying our reparameterization trick for Lie groups (ReLie\footnote{Pronounced `really'.}), followed by a formal proof. Additionally, we provide an implementation section\footnote{Code available at \url{https://github.com/pimdh/relie}} where we study three important examples of Lie groups, deriving the reparameterization details for the $n$-Torus, $\mathbb{T}^N$, the oriented group of 3D rotations, $\SO3$, and the group of 3D rotations and translations, $\SE3$. We conclude by creating complex and multimodal reparameterizable densities on $\SO{3}$ using a novel non-invertible normalizing flow, demonstrating applications of our work in both a supervised and unsupervised setting.
    
\section{PRELIMINARIES} \label{sec:background}
In this section we first cover a number of preliminary concepts that will be used in the rest of this paper.

\subsection{Lie Groups and Lie Algebras}

\paragraph{Lie Group, $G$:}
A \emph{Lie group}, $G$ is a group that is also a smooth manifold.
This means that we can, at least in local regions, describe group elements continuously with parameters. The number of parameters equals the dimension of the group. We can see (connected) Lie groups as continuous symmetries where we can continuously traverse between group elements\footnote{We refer the interested reader to \citep{hall2003lie}.}. Many relevant Lie groups are matrix Lie groups, which can be expressed as a subgroup of the Lie group $\GL{n,\R}$ of invertible square matrices with matrix multiplication as product.

\paragraph{Lie Algebra, $\alg$:}
The \emph{Lie algebra} $\alg$ of a $N$ dimensional Lie group is its tangent space at the identity, which is a vector space of $N$ dimensions. We can see the algebra elements as infinitesimal generators, from which all other elements in the group can be created. For matrix Lie groups we can represent vectors $\bm{v}$ in the tangent space as matrices $\mathbf{v}_\times$.

\paragraph{Exponential Map, $\exp(\cdot)$:}
The structure of the algebra creates a map from an element of the algebra to a vector field on the group manifold.
This gives rise to the \emph{exponential map}, $\exp : \alg \rightarrow G$ which maps an algebra element to the group element at unit length from the identity along the flow of the vector field.
The zero vector is thus mapped to the identity.
For compact connected Lie groups, such as $\SO3$, the exponential map is surjective.
Often, the map is not injective, so the inverse, the $\log$ map, is multi-valued. 
The exponential map of matrix Lie groups is the matrix exponential.

\paragraph{Adjoint Representation, $\ad_\x$:}
The Lie algebra is equipped with with a bracket $[\cdot,\cdot]: \alg \times \alg \to \alg$, which is bilinear. The bracket relates the structure of the group to structure on the algebra. For example, $\log (\exp(x)\exp(y))$ can be expressed in terms of the bracket. The bracket of matrix Lie groups is the commutator of the algebra matrices. The adjoint representation of $x \in \alg$ is the matrix representation of the linear map $\ad_\x :  \alg \to \alg : y \mapsto [x,y]$.

\begin{figure*}[t]
    \centering
    \includegraphics[width=0.7\textwidth]{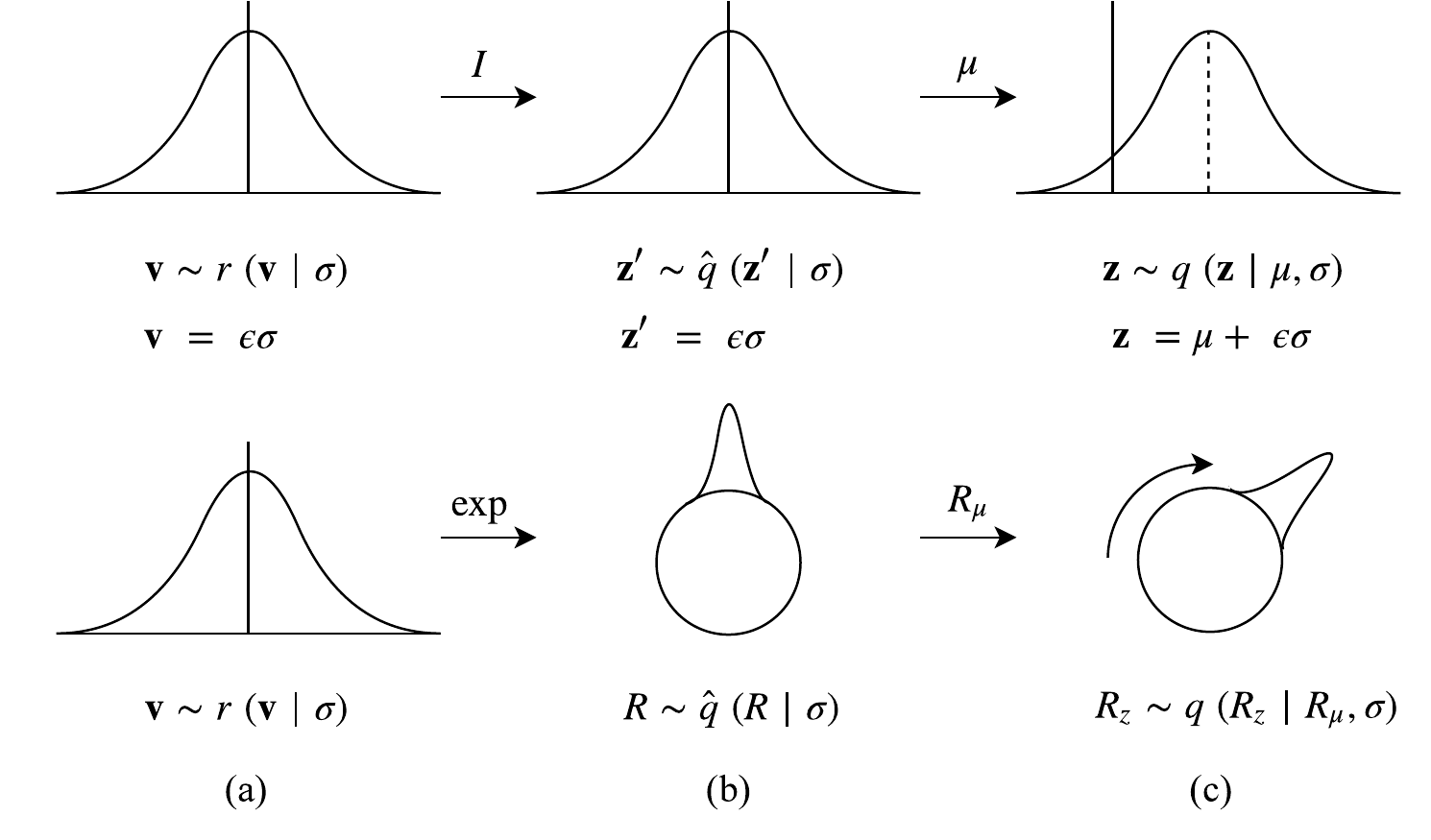}
    \caption{Illustration of reparameterization trick of a Lie group v. the classic reparameterization trick.}
    \label{fig:reparam-pipeline}
\end{figure*}

\subsection{Reparameterization Trick}

The reparameterization trick \citep{price1958useful, bonnet1964transformations, salimans2013fixed, kingma-vae13, pmlr-v32-rezende14} is a technique to simulate samples $ z \sim q( z,\theta)$ as 
$z = \mathcal{T}(\epsilon;\theta)$, where $\epsilon \sim s(\epsilon)$ is independent from $\theta$
\footnote{At most weakly dependent \citep{generalized_reparam_ruiz}.},
and the transformation $\mathcal{T}(\epsilon;\theta)$ should be differentiable w.r.t. $\theta$. It has been shown that this generally results in lower variance estimates than score function variants, thus leading to more efficient and better convergence results \citep{titsias2014doubly, fan2015fast}. This reparameterization of samples $z$, allows expectations w.r.t. $q(z;\theta)$ to be rewritten as $\E_{q(z;\theta)}[f(z)] =\E_{s(\epsilon)}[f(\mathcal{T}(\epsilon;\theta))]$, thus making it possible to directly optimize the parameters of a probability distribution through backpropagation.  

Unfortunately, there exists no general approach to defining a reparameterization scheme for arbitrary distributions. Although there has been a significant amount of research into finding ways to extend or generalize the reparameterization trick \citep{generalized_reparam_ruiz, accept_reject_reparam, implicit_reparam_figurnov},
to the best of our knowledge no such trick exists for spaces with non-trivial topologies such as Lie groups. 

\section{REPARAMETERIZING DISTRIBUTIONS ON LIE GROUPS} \label{sec:reparam-lie}
  
In this section we will first explain our reparameterization trick for distributions on Lie groups (ReLie), by analogy to the classic Gaussian example described in \citep{kingma-vae13}, as we can consider $\mathbb{R}^N$ under addition as a Lie group with Lie algebra $\mathbb{R}^N$ itself. In the remainder we build an intuition for our general theory drawing both from geometrical as well as measure theoretical concepts, concluded by stating our formal theorem.

\subsection{Reparameterization Steps}
The following reparameterization steps (a), (b), (c) are illustrated in Figure \ref{fig:reparam-pipeline}.

(a) We first sample from a reparameterizable distribution $r(\vv | \sigma)$ on $\alg$. Since the Lie algebra is a real vector space, if we fix a basis this is equivalent to sampling a reparameterizable distribution from $\mathbb{R}^N$. In fact, the basis induces an isomorphism between the Lie algebra and $\mathbb{R}^N$ (see Appendix \ref{app:choice-scalar-algebra}).

(b) Next we apply the exponential map to $\vv$, to obtain an element, $g \sim \hat{q}(g | \sigma)$ of the group. If the distribution $r(\vv | \sigma)$ is concentrated around the origin, then the distribution of $\hat{q}(g | \sigma)$ will be concentrated around the group identity. In the Gaussian example on $\mathbb{R}^N$, this step corresponds to the identity operation, and $r = \hat{q}$. As this transformation is in general not the identity operation, we have to account for the possible change in volume using the change of variable formula\footnote{In a sense, this is similar to the idea underlying \emph{normalizing flows} \citep{normalizing-flows}}. Additionally the exponential map is not necessarily injective, such that multiple points in the algebra can map to the same element in the group. We will have a more in depth discussion of both complications in the following subsection.

(c) Finally, to change the location of the distribution $\hat{q}$, we left multiply $g$ by another group element $g_\mu$, applying the group specific operation. In the classic case this corresponds to a translation by $\mu$. If the exponential map is surjective (like in all compact and connected Lie groups), then $g_\mu$ can also be parameterized by the exponential map\footnote{Care must be taken however when $g_\mu$ is predicted by a neural network to avoid homeomorphism conflicts as explored in \citep{falorsi2018explorations, dehaan18homeo-auto}}.

\begin{figure*}[!ht]
    \centering
    \subfigure[$f_*(m')$ no density]{
    \includegraphics[width=0.35\textwidth]{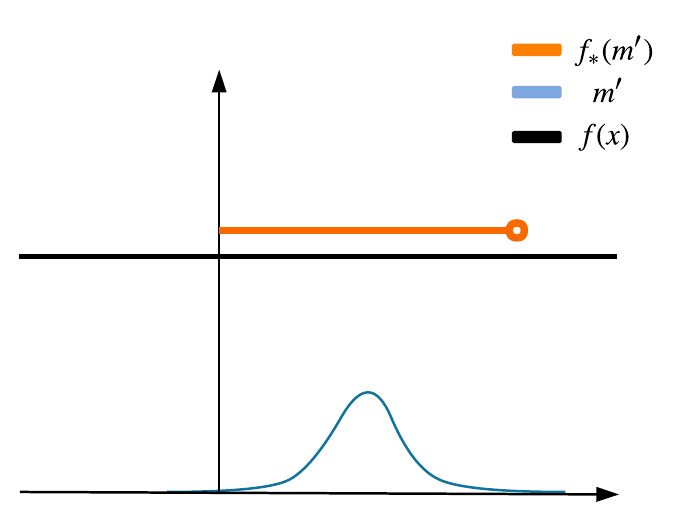}
    \label{fig:non-inject-no-density}
    } 
    \hspace{6em}
    \subfigure[$g_*(m')$ with density]{
    \includegraphics[width=0.35\textwidth]{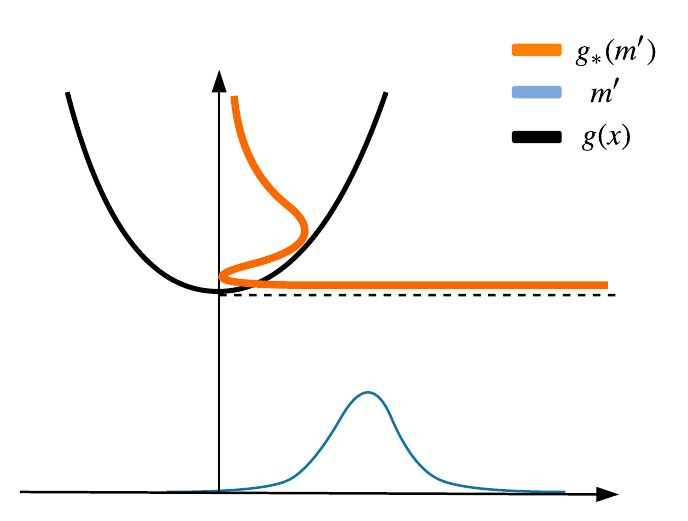}
    \label{fig:non-inject-w-density}
    }
    \caption{Example of the two non-injective mappings, $f(x) = 1$ and $g(x) = x^2 + c$, where the blue line denotes the initial Gaussian density of $\measurealg'$, and the orange line the transformed density. In (a) the pushforward measure of $\measurealg'$ by $f$ collapses to $\delta_1$, while for (b) the pushforward measure by $g$ has a density.} 
    \label{fig:non-injective-example}
\end{figure*}

\subsection{Theory}

\paragraph{Geometrical Concepts} 
When trying to visualize the change in volume, moving from the Lie algebra space to that of the group manifold, we quickly reach the limits of our geometrical intuition. As concepts like volume and distance are no longer intuitively defined, naturally our treatment of integrals and probability densities should be reinspected as well. In mathematics these concepts are formally treated in the fields of differential and Riemannian geometry. To gain insight into building quantitative models of the above-mentioned concepts, these fields start from the local space behavior instead. This is done through the notion of the Riemannian metric, which formally corresponds to "attaching" to the tangent space $T_p G$ at every point $p$ a scalar product $\langle\ ,\ \rangle_p $. This allows to measure quantities like length and angles, and to define a local volume element, in small infinitesimal scales. Extrapolating from this approach we are now equipped to measure sets and integrate functions, which corresponds to having a measure on the space\footnote{We refer the interested reader to \citep{LeeSmooth}.}.
Notice that this measure arose directly from the geometric properties defined by the Riemannian metric. By carefully choosing this metric, we can endow our space with some desirable properties. A standard choice for Lie groups is to use a left invariant metric, which automatically induces a left invariant measure $\nu$, called the \emph{Haar measure} (unique up to a constant):
\begin{align}
    \nu(gE) = \nu(E), \;\;\; \forall g \in G, E \in \mathcal{B}[G], \nonumber
\end{align}
where $gE$ is the set obtained by applying the group element to each element in the set $E$. More intuitively, this implies that left multiplication doesn't change volume.

\paragraph{Measure Theoretical Concepts} Perhaps a more natural way to view this problem comes from measure theory, as we're trying to push a measure on $\alg$, to a space $G$ with a possibly different topology. Whenever discussing densities such as $r$ in \Rn, it is implicitly stated that we consider a density w.r.t. the Lebesgue measure $\lebesgue$. What this really means is that we are considering a measure \ms, absolutely continuous (a.c.) w.r.t. $\lebesgue$, written as $\measurealg \ll \lebesgue$\footnote{See definition \ref{app:def-ac}, Appendix \ref{app:prereq}}. Critically, this is equivalent to stating there exists a density $r$, such that
\begin{equation} \label{eq:ms-def}
    \measurealg(E) = \int_E r\; \text{d} \lebesgue, \quad \forall E \in \mathcal{B}(\mathbb{R}^N), \nonumber
\end{equation}
where $\mathcal{B}(\mathbb{R}^N)$ is the Borel $\sigma$-algebra, i.e. the collection of all measurable sets. When applying the exponential map, we define a new measure on $G$\footnote{We can do this since the exponential map is differentiable, thus continuous, thus measurable.}, technically called the \emph{pushforward} measure, $\exp_*(\measurealg)$\footnote{See definition \ref{app:def-pushforward}, Appendix \ref{app:prereq}}. However, $G$ already comes equipped with another measure $\nu$, not necessarily equal to $\exp_*(\measurealg)$. Hence, if we consider a prior distribution that has a density on $\nu$, in order to compute quantities such as the Kullback-Leibler divergence we also need $\exp_*(\measurealg) \ll \nu$, meaning it has a density $\hat{q}$ w.r.t. $\nu$.

In the case the exponential map is injective, it can easily be shown that the pushforward measure has a density on $\nu$\footnote{In fact, as discussed before we can always reduce to this case by defining a measure with limited support.}. However, that these requirements are not \emph{necessarily} fulfilled can be best explained through a simple example: Consider $f: \mathbb{R} \to \mathbb{R}$, s.t. $f(x) = 1, \forall x \in \mathbb{R}$, this function is clearly differentiable (see Fig. \ref{fig:non-inject-no-density}). 
If we take a measure $\measurealg'$, with a Gaussian density, the pushforward of $\measurealg'$ by $f$ is a Dirac delta, $\delta_1$, for which it no longer holds that $f_*(\measurealg') \ll \lebesgue$. Intuitively, this happens because $f$ is not injective since all points $x \in \mathbb{R}$ are mapped to $1$, such that all the mass of the pushforward measure is concentrated on a single point.

Yet, this does not mean that all non-injective mappings can not be used. Instead, consider $g: \mathbb{R} \to \mathbb{R}$, s.t. $g(x) = x^2 + c, \forall x \in \mathbb{R}$ with $c \in \mathbb{R}$ a constant, and $\measurealg'$ as before (see Fig. \ref{fig:non-inject-w-density}). Although $g$ is clearly not injective, for the pushforward measure by $g$ we still have $g_*(\measurealg') \ll \lebesgue$. The key property here, is that it's possible to partition the domain of $g$ into the sets $(-\infty, 0) \cup (0, \infty) \cup \{0\}$. For the first two, we can now apply the change of variable method on each, as $g$ is injective when restricted to either. The zero set can be ignored, since it has Lebesgue measure 0. This partition-idea can be generally extended for Lie groups, by proving that the Lie algebra domain can be partitioned in a set of measure zero and a countable union of open sets in which the exponential map is a 
diffeomorphism. This insight proven in Lemma \ref{lemma:local-diffeomorphism}, allows us to prove the general theorem:

\begin{theorem} \label{thm:main-text}
Let $G$, $\alg$, $\measurealg$, $\lebesgue$, $\nu$ be defined as above, then $\exp_*(\measurealg)\ll \nu$ with density:
\begin{equation}
    p(a) = \sum_{{\x\in \alg : \exp(\x) = a}} r(\x) |J(\x)|^{-1}, \quad a\in G,
\end{equation}
where $J(\x):= \det \left(\sum_{k=0}^\infty \dfrac{(-1)^k}{(k+1)!} (\ad_\x)^k \right)$ 
\end{theorem}
\begin{proof}
\text{See Appendix \ref{thm:main}}
\end{proof}

\paragraph{Location Transformation}
Having verified the pushforward measure has a density, $\hat{q}$, the final step is to recenter the location of the resulting distribution. In practice, this is done by left multiplying the samples by another group element. Technically, this corresponds to applying the left multiplication map
\begin{align}
    L_{g_\mu}: G \to G, \;\; g \mapsto g_\mu g \nonumber
\end{align}
Since this map is a diffeomorphism, we can again apply the change of variable method. Moreover, if the measure on the group $\nu$ is chosen to be the Haar measure, as noted before applying the left multiplication map leaves the Haar measure unchanged. In this case the final sample thus has density
\begin{align}
    g_z \sim q(g_z | g_\mu, \sigma ) = \hat q(g_\mu^{-1}g_z|\sigma ) \nonumber
\end{align}
Additionally, the entropy of the distribution is invariant w.r.t. left multiplication.

\section{IMPLEMENTATION}\label{sec:prac-guide}
In this section we present general implementation details, as well as worked out examples for three interesting and often used groups: the n-Torus, $\mathbb{T}^N$, the oriented group of 3D rotations, $\SO{3}$, and the 3D rotation-translation group, $\SE{3}$. The worst case reparameterization computational complexity for a matrix Lie group of dimension $n$ can be shown\footnote{See Appendix \ref{app:computing-jacobian}.} to be $O(n^3)$. However for many Lie groups closed form expressions can be derived, drastically reducing the complexity.

\subsection{Computing \texorpdfstring{$J(\x)$}{J(x)}}
% \subsection{Computing $J(\x)$}
The term $J(\x)$ as appearing in the general reparameterization theorem \ref{thm:main-text}, is crucial to compute the change of volume when pushing a density from the algebra to the group. Here, we given an intuitive explanation for Matrix Lie groups in $d$ dimensions with matrices of size $n \times n$. For a formal general explanation, we refer to Appendix \ref{app:change-of-varables}.

The image of the $\exp$ map is the $d$ dimensional manifold, Lie group $G$, embedded in $\mathbb{R}^{n \times n}$. An infinitesimal variation in the input around point $\x \in \alg$, creates an infinitesimal variation of the output, which is restricted to the $d$ dimensional manifold $G$. Infinitesimally this gives rise to a linear map between the tangent spaces at input and output. This is the Jacobian.

The change of volume is the determinant of the Jacobian. To compute it, we express the tangent space at the output in terms of the chosen basis of the Lie algebra. This is possible, since a basis for a Lie algebra provides a unique basis for the tangent space throughout $G$. This can be computed analytically for any $\x$, since the $\exp$ map of matrix Lie groups is the matrix exponential, for which derivatives are computable. Nevertheless a general expression of $J(\x)$ exists for any Lie Group and is given in terms of the complex eigenvalue spectrum $Sp(\cdot)$ of the adjoint representation of $\x$, which is a linear map:

\begin{theorem}
Let $G$ be a Lie Group and $\alg$ its Lie algebra, then it can be shown that $J(\x)$ can be computed using the following expression
\begin{equation}
    J(\x):=
    \prod_{\substack{\lambda \in \text Sp(\ad_\x) \\\lambda \neq 0 }} \frac{\lambda}{1 - e^{-\lambda}}
\end{equation} 
\end{theorem}

\begin{proof}
    \text{See Appendix \ref{app:computing-jacobian}}
\end{proof}

\subsection{Three Lie Group Examples}\label{sec:examples}

\paragraph{The $n$-Torus, $\mathbb{T}^N$:}
The $n$-Torus is the cross-product of $n$ times $\mathcal{S}^1$. It is an abelian (commutative) group, which is interesting to consider as it forms an important building block in the theory of Lie groups. The $n$-Torus has the following matrix representation:
\begin{equation}
T(\bm\alpha) := 
\begin{bmatrix}
    \diagentry{B_{\alpha_1}}\\
    &\diagentry{\xddots}\\
    &&\diagentry{B_{\alpha_n}}\\
\end{bmatrix}, B_\alpha := \begin{bmatrix}
   \cos \alpha & -\sin \alpha \\
   \sin \alpha & \cos \alpha
    \end{bmatrix}, \nonumber
\end{equation}
where $\bm{\alpha} = (\alpha_1,\cdots, \alpha_n)\in \mathbb{R}^n$. The basis of the Lie algebra is composed of $2n\times2n$ block-diagonal matrices with $2\times2$ blocks s.t. all blocks are $0$ except one that is equal to $L$: 
\begin{equation}
    L(\bm{\alpha}) = \begin{bmatrix}
        \diagentry{\alpha_1 L}\\
        &\diagentry{\xddots}\\
        &&\diagentry{\alpha_n L}\\
    \end{bmatrix}, \quad
     L := \begin{bmatrix}
            0 &-1 \\
            1 & 0
        \end{bmatrix} \nonumber
\end{equation}
The exponential map is s.t. the pre-image can be defined from the following relationship $L(\bm \alpha) \mapsto T(\bm \alpha)$:
\begin{equation}
    \exp(L(\bm \alpha + 2\pi\bm k)) = \exp(L(\bm \alpha)), \quad \bm k \in \mathbb{Z}^n \nonumber
\end{equation}

The pushforward density is defined as
\begin{align} \label{eq:jx-torus}
    J(L(\bm \alpha)) &= 1 \nonumber\\
    \hat q(T(\bm \alpha)|\sigma) &= \sum_{\bm k \in \mathbb{Z}^n} r\left(\bm \alpha + 2\bm k\pi | \sigma\right)
\end{align}
It can be observed that there is no change in volume. The resulting distribution on the circle or 1-Torus, which is also the Lie group SO(2), is illustrated in Appendix~\ref{app:wrap}.

\paragraph{The Special Orthogonal Group, $\SO{3}$:}
The Lie group of orientation preserving three dimensional rotations has its matrix representation defined as
\begin{equation} \label{eq:so3-def}
    \SO3:= \{R \in \GL{3, \mathbb{R}} : R^\top R = I \land \det(R) = 1 \} \nonumber  
\end{equation}
The elements of its Lie algebra $\mathfrak{so}(3)$, are represented by the 3D vector space of skew-symmetric $3 \times 3$ matrices. We choose a basis for the Lie algebra:
\begin{align}
   \hspace{-1.4em}L_{1,2,3} := \begin{bmatrix} 0 & 0 & 0 \\ 0 & 0 & -1 \\ 0 & 1 & 0 \end{bmatrix}, 
   \begin{bmatrix} 0 & 0 & 1 \\ 0 & 0 & 0 \\ -1 & 0 & 0 \end{bmatrix}, 
   \begin{bmatrix} 0 & -1 & 0 \\ 1 & 0 & 0 \\ 0 & 0 & 0 \end{bmatrix} \nonumber
\end{align}
This provides a vector space isomorphism between $\mathbb{R}^3$ and $\mathfrak{so(3)}$, written as $[\;\cdot\;]_\times : \mathbb{R}^3 \to  \mathfrak{so(3)}$. Assuming the decomposition $\mathbf{v}_\times = \theta \mathbf{u}_\times$, s.t. $\theta \in \mathbb{R}_{\ge 0}, \; \|\mathbf{u} \| =1$, the exponential map is given by the Rodrigues rotation formula 
\citep{rodrigues1840lois}
\begin{align} \label{eq:rodrigues}
    \exp(\mathbf{v}_\times) = \mathbf{I} + \sin(\theta)\mathbf{u}_\times + 
    (1 - \cos(\theta))\mathbf{u}_\times^2
\end{align}
Since $\SO3$ is a compact and connected Lie group this map is surjective, however it is not injective. The complete preimage of an arbitrary group element can be defined by first using the principle branch $\log(\cdot)$ operator to find the unique Lie algebra element next to the origin, and then observing the following relation
 \begin{equation}
     \exp(\theta \mathbf{u}_\times) =  \exp((\theta+2k\pi)\mathbf{u}_\times) \quad k\in \mathbb{Z} \nonumber
 \end{equation}
In practice, we will already have access to such an element of the Lie algebra due to the sampling approach. The pushforward density defined almost everywhere as
\begin{align} \label{eq:jx-so3}
    J(\vv) &= \frac{\|\vv\|^2}{2 - 2\cos\|\vv\|}\\
    \hat q(R|\sigma) &= \sum_{k \in \mathbb{Z}} r\left(\frac{\log(R)}{\theta(R)}(\theta(R) + 2k\pi) \bigg| \sigma\right)\frac{(\theta(R) + 2k\pi)^2}{3 - \text{tr}({R})} \nonumber,
\end{align} 
where $R \in \SO{3}$ and 
\begin{align} 
    \theta({R}) = \|\log(R)\| = \cos^{-1}\left(\frac{\text{tr}(R)-1}{2}\right) \nonumber
\end{align}
\begin{figure}[t]
    \centering
    \subfigure[$\SO{3}$]{
    \includegraphics[width=0.39\textwidth]{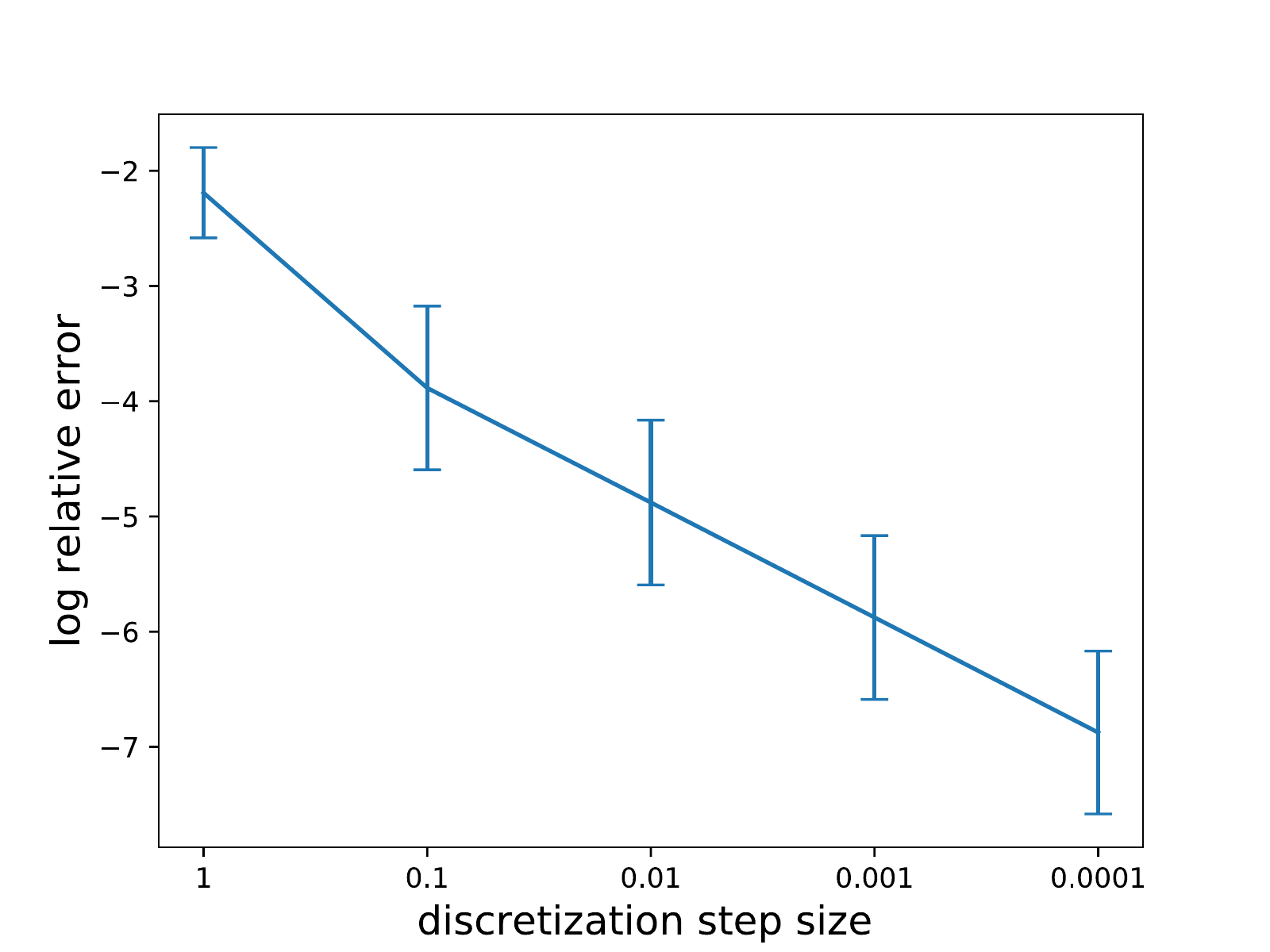}
    \label{fig:so3-jacobian}
    } 
    \subfigure[$\SE{3}$]{
    \includegraphics[width=0.39\textwidth]{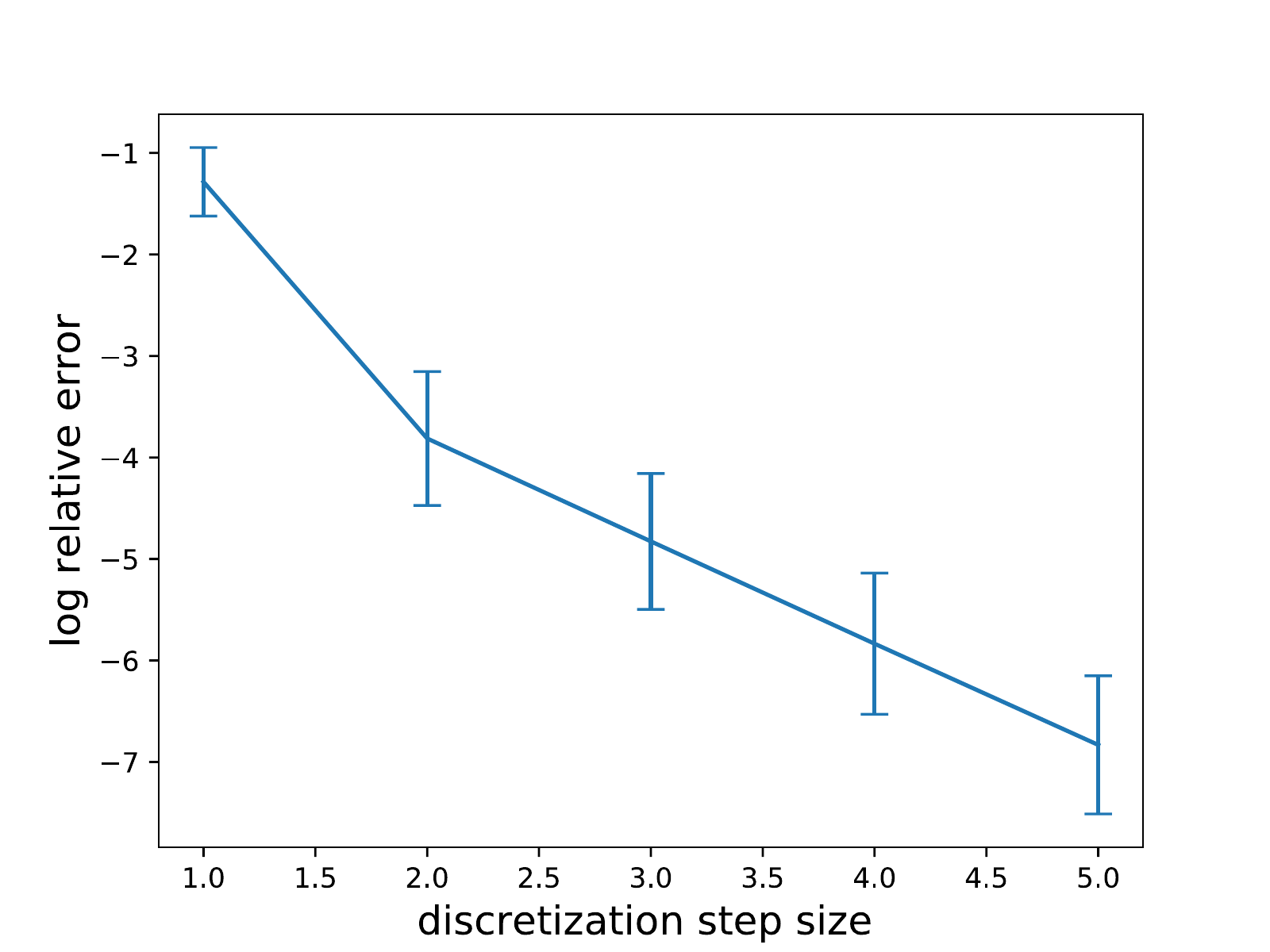}
    \label{fig:se3-jacobian}
    }
    \caption{Plotted the log relative error of analytical and numerical estimations of $J(\x)$ for the groups $\SO{3}$ and $\SE{3}$,
    equations 
    \eqref{eq:jx-so3}, \eqref{eq:jx-se3}.
    Numerical estimation of Jacobian performed by taking small discrete steps of decreasing size ($x$-axis) in each Lie algebra direction. The error is evaluated at $1000$ randomly sampled points.}
    \label{fig:jacobian-numerical-checks}
\end{figure}
\paragraph{The Special Euclidean Group, $\SE{3}$:} This Lie group extends $\SO{3}$ by also adding translations. Its matrix representation is given by
\begin{equation}
    \begin{bmatrix}
        R & {\bf u} \\ 
        0 & 1
    \end{bmatrix}
    , \quad {\bf u} \in \mathbb{R}^3, R \in \SO{3} \nonumber
\end{equation}
The Lie algebra, $\mathfrak{se(3)}$ is similarly built concatenating a skew-symmetric matrix and a $\mathbb{R}^3$ vector
\begin{equation}
S({\bm \omega}, {\bf u }) = 
    \begin{bmatrix}
        \bm\omega_\times & \bf{u} \\ 
        0 & 0
    \end{bmatrix}
    , \quad \bf{u},\;{\bm \omega} \in \mathbb{R}^3, \nonumber
\end{equation}
A basis can easily be found combining the basis elements for $\mathfrak{so(3)}$ and the canonical basis of $\mathbb{R}^3$. The exponential map from algebra to group is defined as
\begin{equation}
    \begin{bmatrix}
        \bm{\omega}_\times & {\bf u} \\ 
        0 & 0
    \end{bmatrix}
    \mapsto\begin{bmatrix}
        \exp(\bm\omega_\times) & V {\bf u} \\ 
        0 & 1
    \end{bmatrix}, \nonumber
\end{equation}
where $\exp(\cdot)$ is defined as in equation \eqref{eq:rodrigues},
and
\begin{equation}
V = I + \left(\frac{1- \cos(\|\bm\omega\|)}{\|\bm\omega\|^2}\right)\bm\omega_\times + \left(\frac{\|\bm\omega\|- \sin(\|\bm\omega\|)}{\|\bm\omega\|^3}\right)\bm\omega_\times^2 \nonumber
\end{equation}
From the expression of the exponential map it is clear that the preimage can be described similar to $\SO{3}$. Finally the pushforward density is defined almost everywhere as
\begin{align} \label{eq:jx-se3}
   &\;\;\;\;J(S(\bm\omega, {\bf u})) = \left[\frac{\|\bm\omega\|^2}{2 - 2\cos\|\bm\omega\|} \right]^2 
    \\
    &\hat q(M|\sigma) = \nonumber 
    \\
    &\sum_{k \in \mathbb{Z}} r \left(S \left( \frac{\bm  \omega}{\|\bm\omega\|} (\|\bm\omega\| + 2k\pi),{\bf u} \right) \bigg| \sigma \right) \left[\frac{\|\bm\omega + 2k\pi\|^2}{2 - 2\cos\|\bm\omega\|} \right]^2, \nonumber
\end{align}
where $\omega$ and $u$ such that $S(\bm \omega, {\bf u}) = \log(M)$. The $\log$ can be easily defined from the $\log$ in $\SO{3}$.

\section{RELATED WORK} \label{sec:related-work}

Various work has been done in extending the reparameterization trick to an ever growing amount of variational families. \citet{implicit_reparam_figurnov} provide a detailed overview, classifying existing approaches into (1) finding \emph{surrogate distributions}, which in the absence of a reparameterization trick for the desired distribution, attempts to use an acceptable alternative distribution that \emph{can} be reparameterized instead \citep{stick}.
(2) \emph{Implicit reparameterization gradients},
or \emph{pathwise} gradients,
introduced in machine learning by \citet{salimans2013fixed}, extended by \citet{graves2016stochastic}, and later generalized by \citet{implicit_reparam_figurnov} using implicit differentiation. (3) \emph{Generalized reparameterizations} finally try to generalize the standard approach as described in the preliminaries section. Notable are \citep{generalized_reparam_ruiz}, which relies on defining a suitable invertible standardization function to allow a weak dependence between the noise distribution and the parameters, and the closely related \citep{accept_reject_reparam} focusing on rejection sampling.

All of the techniques above can be used orthogonal to our approach, by defining different distributions over the Lie algebra. While some even allow for reparameterizable densities on spaces with non-trivial topologies\footnote{For example \citet{s-vae} reparameterize the von Mises-Fisher distribution which is defined on $\mathcal{S}^M$, with $\mathcal{S}^1$ isomorphic to the Lie group $\SO{2}$.}, none of them provide the tools to correctly take into account the volume change resulting from pushing densities defined on $\mathbb{R}^N$ to arbitrary Lie groups. In that regard the ideas underlying \emph{normalizing flows} (NF) \citep{normalizing-flows} are the closest to our approach, in which probability densities become increasingly complex through the use of injective maps. Two crucial differences however with our problem domain, are that the change of variable computation now needs to take into account a transformation of the underlying space, as well as the fact that the exponential map is generally not injective. NF can be combined with our work to create complex distributions on Lie groups, as is demonstrated in the next section.

\begin{figure}[t]
    \centering
    \subfigure[Line symmetry]{
    \includegraphics[width=.37\textwidth]{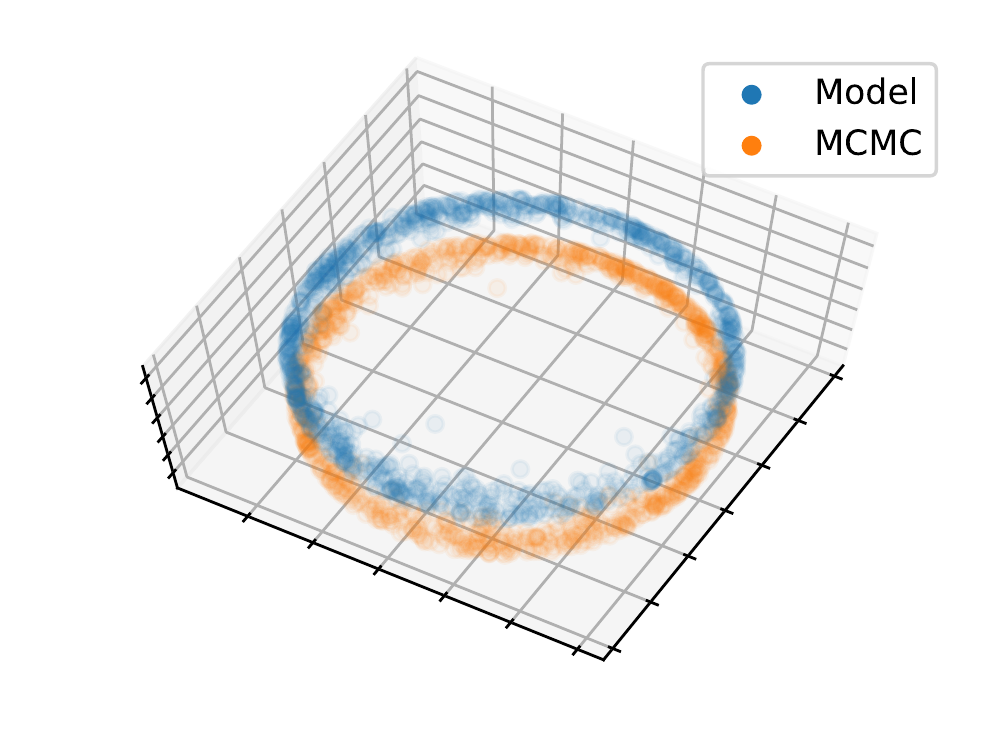}
    }
    \subfigure[Triangular symmetry]{
    \includegraphics[width=.37\textwidth]{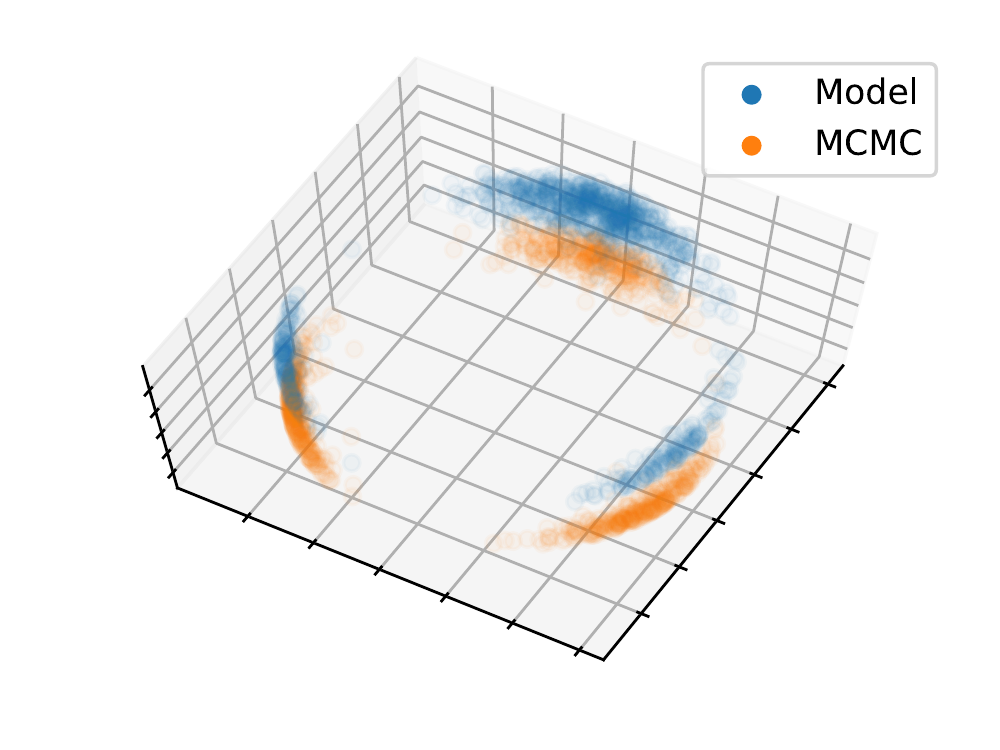}
    }
    \caption{Samples of the Variational Inference model and Markov Chain Monte Carlo of Experiment \ref{exp:vi}. Outputs are shifted in the z-dimension for clarity.}
    \label{fig:unsupervised}
\end{figure}

\begin{figure*}[t]
    \centering
    \includegraphics[width=0.85\textwidth]{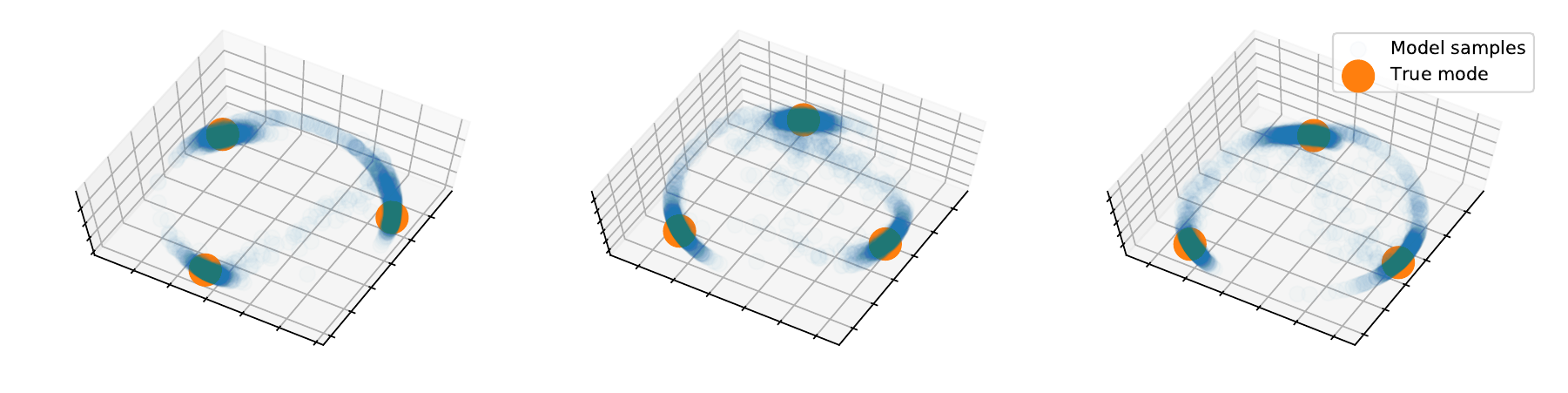}
    \caption{Samples from conditional $\SO{3}$ distribution $p(g|x)$ for different $x$, where $x$ has the symmetry of rotations of $\sfrac{2\pi}{3}$ along one axis of Experiment \ref{exp:supervised}. Shown are PCA embeddings of the matrices, learned using Locally Invertible Flow (LI-Flow) with Maximum Likelihood Estimation.}
    \label{fig:supervised}
\end{figure*}

Defining and working with distributions on homogeneous spaces, including Lie groups, was previously investigated in \citep{chirikjian2000engineering, chirikjian2010information, wolfe2011bayesian, Chirikjian2012-gr,  Chirikjian2016-ch, ming2018variational}.
\citet{barfoot2014associating} also discuss implicitly defining distributions on Lie groups, through a distribution on the algebra, focusing on the case of $\SE{3}$. However, these works only consider the neighbourhood of the identity, making the exponential map injective, but the distribution less expressive. In addition, generally only Gaussian distributions on the Lie Algebra are used in past work.
\citet{cohen2015harmonic} devised harmonic exponential families which are a powerful family of distributions defined on homogeneous spaces. These works all did not concentrate on making the distributions reparameterizable. \citet{mallasto2018wrapped} defined a wrapped Gaussian process on Riemannian manifolds  through the pushforward of the $\exp$ map without providing an expression for the density.

\section{EXPERIMENTS} \label{sec:experiments}
We conduct two experiments on $\SO{3}$ to highlight the potential of using complex and multimodal reparameterizable densities on Lie groups\footnote{See Appendix \ref{app:exp-additional-details} for additional details.}.

\paragraph{Normalizing Flow}
To construct multimodal distributions Normalizing Flows are used \citep{dinh2015nice, normalizing-flows}:
\begin{align*}
    \mathbb{R}^d \overset{f}{\longrightarrow} \mathbb{R}^d \overset{r \cdot \tanh}{\longrightarrow} \mathbb{R}^d \cong \alg \overset{\exp}{\longrightarrow} G,
\end{align*}
where $f$ is an invertible Neural Network consisting of several coupling layers \citep{dinh2015nice}, the $\tanh(\cdot)$ function is applied to the norm and a unit Gaussian is used as initial distribution. The hyperparameter $r$ determines the non-injectivity of the exp map and thus of the flow. $r$ must be chosen such that  the image of $r\cdot \tanh$ is contained in the regular region of the $\exp$ map. For sufficiently small $r$, the entire flow is invertible, but may not be surjective, while for bigger $r$ the flow is non-injective, with a finite inverse set at each $g \in G$, as $\exp$ is a local diffeomorphism and the image of $r\cdot \tanh$ has compact support. For details see Appendix \ref{app:main-theorem}.
For such a \emph{Locally Invertible Flow} (LI-Flow), the likelihood evaluation requires us to branch at the non-injective function and traverse the flow backwards for each element in the preimage.

\subsection{Variational Inference}\label{exp:vi}

In this experiment we estimate the \SO{3} group actions that leave a symmetrical object invariant. This highlights how our method can be used in probabilistic generative models and unsupervised learning tasks. We have a generative model $p(x|g)$ and a uniform prior over the latent variable $g$. Using Variational Inference we optimize the Evidence Lower Bound to infer an approximate posterior $q(g|x)$ modeled with LI-Flow.

Results are shown in Fig. \ref{fig:unsupervised} and compared to Markov Chain Monte Carlo samples. We observe the symmetries are correctly inferred.

\subsection{Maximum Likelihood Estimation}\label{exp:supervised}
To demonstrate the versatility of the reparameterizable Lie group distribution, we learn supervised pose estimation by learning a multimodal conditional distribution using MLE, as in \citep{dinh2017density}.

We created data set: $(x, g'=\exp(\epsilon)g)$ of objects $x$ rotated to pose $g$ and algebra noise samples $\epsilon$. The object is symmetric for the subgroup corresponding to rotations of $\sfrac{2\pi}{3}$ along one axis. We train a LI-Flow model by maximizing:
$\mathbb{E}_{x,g'}\log p(g'|x)$. The results in Fig. \ref{fig:supervised} reveal that the LI-Flow successfully learns a multimodal conditional distribution.

\section{CONCLUSION} \label{sec:conclusion}
In this paper we have presented a general framework to reparameterize distributions on Lie groups (ReLie), that enables the extension of previous results in reparameterizable densities to arbitrary Lie groups. Furthermore, our method allows for the creation of complex and multimodal distributions through normalizing flows, for which we defined a novel \emph{Locally Invertible Flow} (LI-Flow) example on the group $\SO{3}$. We empirically showed the necessity of LI-Flows in estimating uncertainty in problems containing discrete or continuous symmetries.

This work provides a bridge to leverage the advantages of using deep learning to estimate uncertainty for numerous application domains in which Lie groups play an important role. In future work we plan on further exploring the directions outlined in our experimental section to more challenging instantiations. Specifically, learning rigid body motions from raw point clouds or modeling environment dynamics for applications in optimal control present exciting possible extensions.  

\clearpage
\newpage

\subsubsection*{Acknowledgements}
The authors would like to thank Rianne van den Berg and Taco Cohen for suggestions and insightful discussions to improve this manuscript. 

\bibliography{main}

\newpage
\newpage
\onecolumn

\appendix
%%%%%%%%%%%%%%%%%%%%%%%%%%%%%%
%           APPENDIX
%%%%%%%%%%%%%%%%%%%%%%%%%%%%%%

\section{EXPERIMENTS: ADDITIONAL DETAILS} \label{app:exp-additional-details}

\subsection{A Small Note on Alternative Proxies} \label{app:a-small-note-on-exps}

The main focus of the experiments in this work is to show how our framework enables the usage of reparameterizable distributions on arbitrary Lie groups in a probabilistic deep learning setting, which to the best of our knowledge is not possible with current alternatives. The experiments therefore represent typical prototypes of applications, which can now be tackled using a general approach. To avoid confusion, it might very well be possible to design specialized one-off solutions for learning distributions on specific Lie groups, however, in this paper we aim at providing a \emph{general} framework for doing this task. 

\subsection{Supplementary Details on VI Experiment} \label{app:exp-supp-vi}

\paragraph{Setup} In this proto-typical Variational Inference experiment we provide an intuitive example of the need for complex distributions in the difficult task of estimating which group actions of $\SO{3}$ leave a symmetrical object invariant. For didactic purposes we take two ordered points, $\x_1, \x_2 \in \R{3}$, and perform LI-Flow VI to learn the approximate posterior over rotations. We evaluate the learned distribution by comparing its samples to those of the true posterior obtained using the Metropolis-Hastings algorithm.

\paragraph{Results} Results are shown in Fig. \ref{fig:unsupervised}. As expected, the discovered distribution over $\SO{3}$ group actions is a rotational subgroup, $\mathcal{S}^1$. Clearly, the learned approximate posterior almost perfectly matches the true posterior. Instead, using a simple centered distribution such as the pushforward of a Gaussian as the variational family, would make learning the observed topology problematic, as all probability mass would focus around a single rotation.

\subsection{Supplementary Details on MLE Experiment} \label{app:exp-supp-mle}

\paragraph{Setup} We generate a random vector $x_0$ that has a linear Lie group action. Then we create a random variable $g \in G$ uniformly distributed representing the pose and a noisy version $g' = \exp(\epsilon)g$ with $\epsilon \sim \mathcal(0, 0.01)$. We observe $x' = g'(x_0)$ and need to predict $g$. This corresponds to having noisy observations of an object $x$ from different poses and needing to estimate the pose $p(g|x')$. When the object is symmetrical, that is a subgroup $D \subset G$ exists such that $d(x_0)=x_0$ for all $d \in D$, $p(g|x')$ should have modes corresponding to the values in $D$.

\paragraph{Results} This is evaluated on $\SO{3}$. The object $x_0$ is taken to an element of the representation space of $\SO{3}$, as in \citep{falorsi2018explorations}. It is made symmetric by taking the average of $\{d(x_0) | d \in D\}$. $D$ is taken to be the cyclic group of order 3 corresponding to rotations of $\sfrac{2\pi}{3}$ along one axis. The results show in Figure \ref{fig:supervised} reveal that the LI Flow successfully learns complicated conditional distributions.

\section{DISTRIBUTIONS ON THE CIRCLE}\label{app:wrap}

As an example of how the reparametrizable distribution on Lie groups behaves in practice, we illustrate in Figure \ref{fig:wrap} the distribution that arises when a univariate Normal distribution is pushed forward to the Lie group $SO(2)$, homeomorphic to the circle, with the exponential map. 

\begin{figure*}[!ht]
    \centering
    \subfigure[$\mathcal{N}(0, 0.5)$]{
    \includegraphics[width=0.3\textwidth]{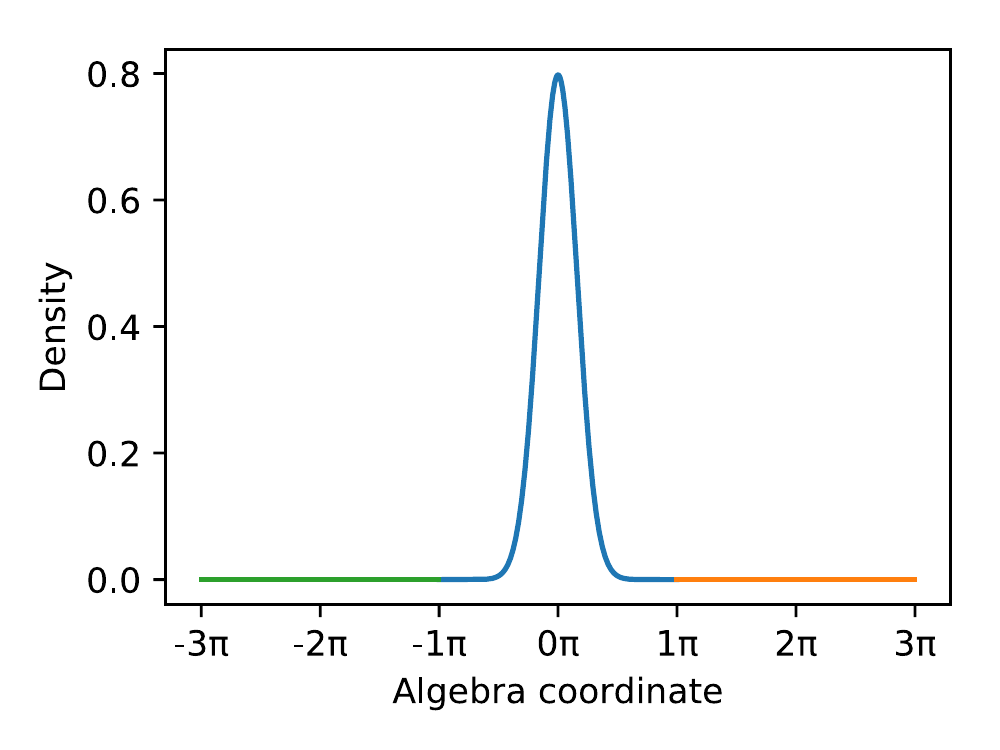}
    } 
    % \hspace{6em}
    \subfigure[$\exp_*\mathcal{N}(0, 0.5)$]{
    \includegraphics[width=0.3\textwidth]{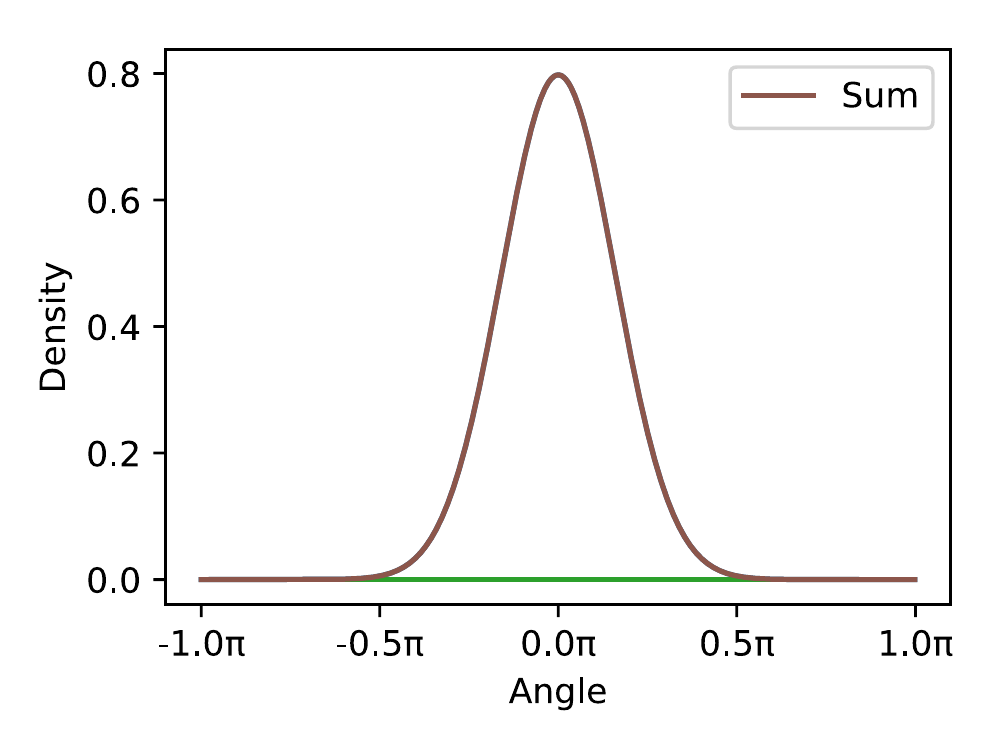}
    } 
    % \hspace{6em}
    \subfigure[Illustration of $\exp_*\mathcal{N}(0, 0.5)$]{
    \includegraphics[width=0.3 \textwidth]{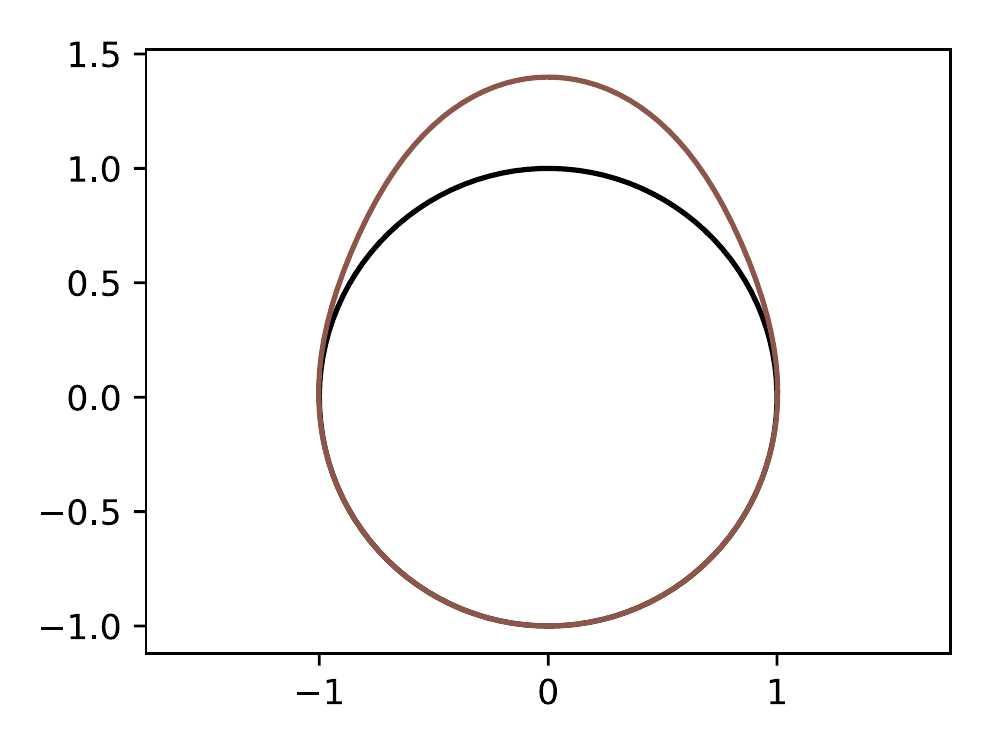}
    }
    \centering
    \subfigure[$\mathcal{N}(0, 2)$]{
    \includegraphics[width=0.3\textwidth]{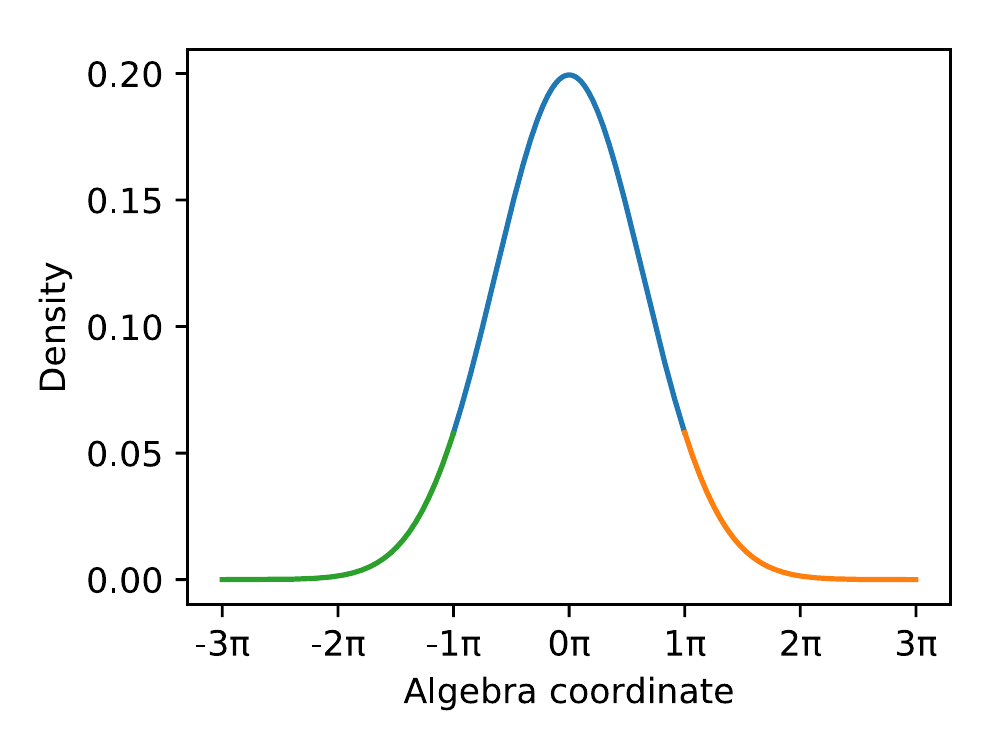}
    } 
    % \hspace{6em}
    \subfigure[$\exp_*\mathcal{N}(0, 2)$]{
    \includegraphics[width=0.3\textwidth]{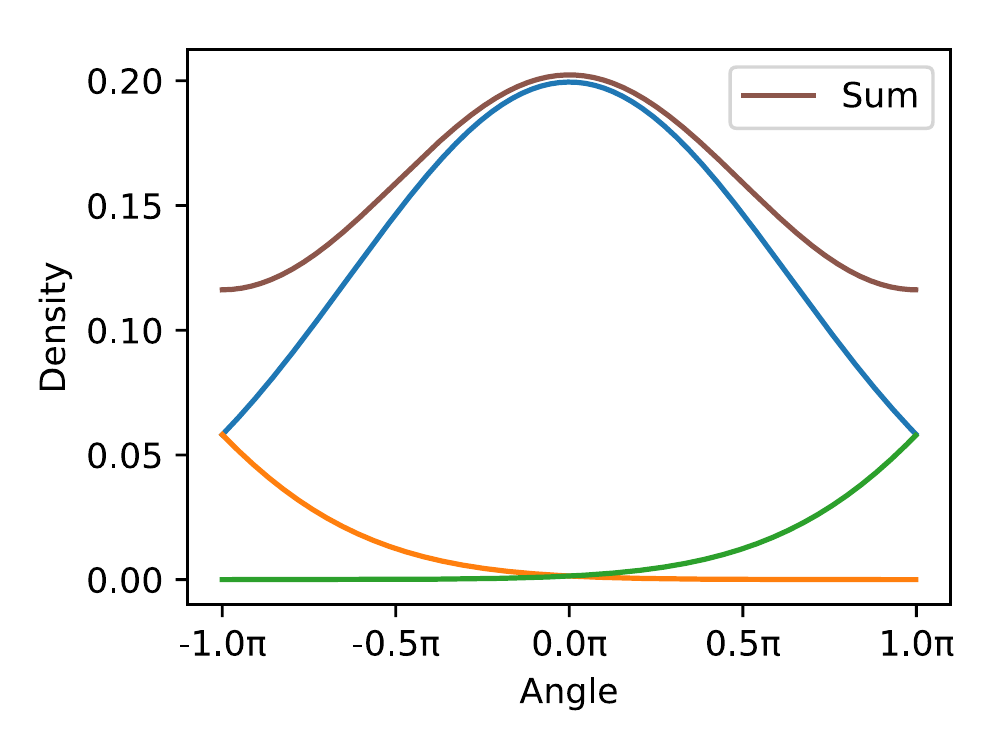}
    } 
    % \hspace{6em}
    \subfigure[Illustration of $\exp_*\mathcal{N}(0, 2)$]{
    \includegraphics[width=0.3 \textwidth]{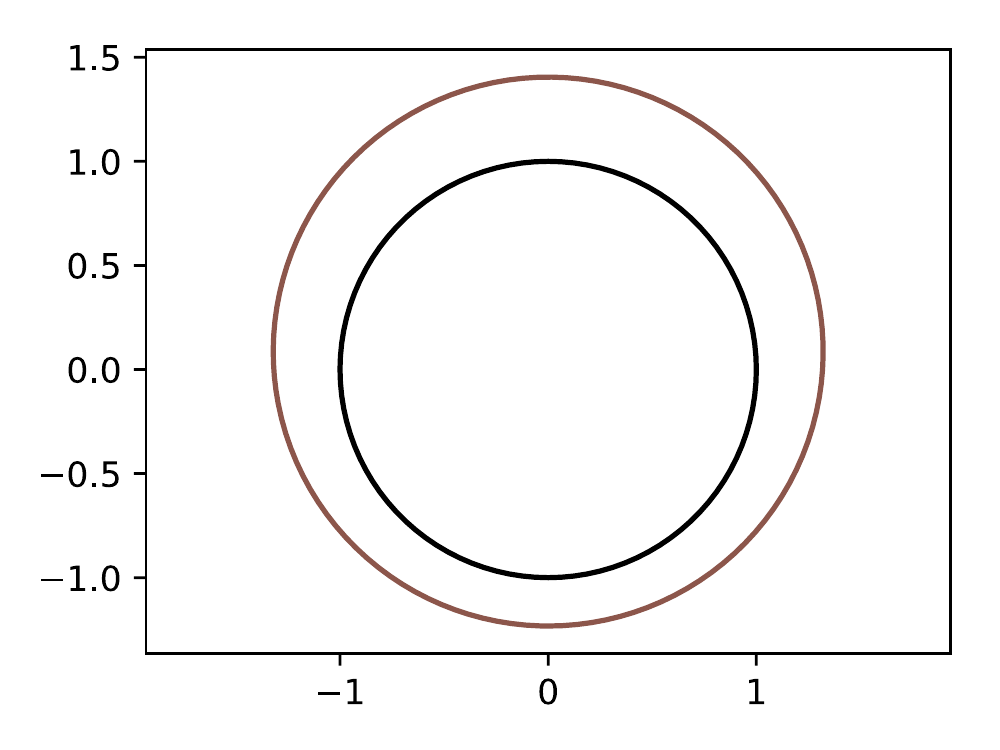}
    }
    \caption{Density of pushforward of Normal distributions with zero mean and scale 0.5 and 2 to the Lie group $SO(2)$.
    Following Equation~\ref{eq:jx-torus}, the density on the group in (b) and (e) at angle $\theta$ is simply the sum of the algebra density of the pre-images of $\theta$.
    The circular representation in (c) and (f) illustrate the density $q$ on the group by drawing a loop with radius $1+q(\theta)$, for angle $\theta$.
    } 
    \label{fig:wrap}
\end{figure*}

\section{PREREQUISITES} \label{app:prereq}
\begin{defn}[Absolutely continuous measures, see \cite{Klenke2014}] \label{app:def-ac}
Let $(X, \mathcal{A})$ be a measurable space, and 
$\nu,\measurealg : \mathcal{A} \to [0, \infty]$ two measures on $(X,\mathcal{A})$. Then $\nu$ is said to be \emph{absolutely continuous} with respect to $\measurealg$, written as $\nu \ll \measurealg$, iff for all $A \in \mathcal{A}$ we have that
\begin{equation}
    \measurealg(A) = 0 \implies \nu(A) = 0.
\end{equation}
\end{defn}
 
\begin{defn}[Density between two measures] \label{app:def-density}
Let $(X, \mathcal{A})$ be a measurable space, and 
$\nu,\measurealg : \mathcal{A} \to [0, \infty]$ two measures on $(X,\mathcal{A})$.
One says that $\nu$ has a \emph{density} w.r.t.\ $\measurealg$ iff there is a measurable function $f:\, X \to \mathbb{R}_{\ge 0}$ such that for all $A \in \mathcal{A}$ we have:
$$ \nu(A) = \int_A f\, d\measurealg.$$
It is knows (see \cite{Klenke2014}) that a density (if existent) is unique up to a $\measurealg$-zero measure and it is often denoted as: $f(x)=\frac{d\nu}{d\measurealg}(x).$
\end{defn}

\begin{theorem}[Radon-Nikod\'ym, see \cite{Klenke2014} Cor.\ 7.34] \label{app:thm-radon-nikodym}
Let $(X, \mathcal{A})$ be a measurable space, and 
$\nu,\measurealg : \mathcal{A} \to [0, \infty]$ two $\sigma$-finite measures on $(X,\mathcal{A})$.
Then one has the equivalence:
$$ \nu \text{ has a density w.r.t.\ } \measurealg 
\,\iff\,  \nu \ll \measurealg.$$
\end{theorem}

\begin{defn}[Pushforward measure] \label{app:def-pushforward}
Let $(X, \mathcal{A}, \measurealg )$ be a measure space, $(Y, \mathcal{B})$ a measurable space, and let $f: X \to Y$ be a measurable map. Then the \emph{pushforward measure} of $\measurealg$ along $f$, in symbols  $f_*\measurealg$, is defined as follows
\begin{equation}
    (f_*\measurealg)(B) := \measurealg(f^{-1}(B)), \quad \text{for } B \in \mathcal{B}.
\end{equation}
\end{defn}

\begin{defn}[The standard measure on (pseudo-)Riemannian manifolds, 
see \citep{nlab:volume_form}]
\label{rem:metric-measure}
Let $(M,g)$ be a (pseudo-)Riemannian manifold with metric tensor $g$. 
The standard measure $\measurealg_g$ on $M$ w.r.t.\ $g$ is in local (oriented) coordinates per definition given by the density $\sqrt{|\det(g)|}$ w.r.t.\ the Lebesgue measure, where $|\det(g)|(x)$ is the absolute value of the determinant of the matrix of $g$ in the local coordinates at point $x$.
Note that the standard measure w.r.t.\ $g$ always exists.
\end{defn}

We are mainly interested in probability distributions on (pseudo-)Riemannian manifolds $(M,g)$ that have a density w.r.t.\ the standard measure $\measurealg_g$ (i.e.\ that are absolute continuous w.r.t. $\measurealg_g$).

\section{CHANGE OF VARIABLES}\label{app:change-of-varables}
Consider a $n$ dimensional Lie group $G$ and its Lie algebra $\alg$. Then a scalar product $\langle\  , \rangle$ on  $\alg$ induces left invariant Riemannian metric on $G$ in the following way: 
\begin{equation}\label{eq:left-metric}
    \langle \x,\y\rangle_a = \langle \dd(L_{a^{-1}})_a \x, \dd(L_{a^{-1}})_a \y \rangle
    \quad \forall a\in G, \quad \x,\y \in T_aG
\end{equation}

Where $\dd(L_{a^{-1}})_a : T_aG \to T_eG \cong \alg $ is the differential of the Left action by $a^{-1}$.
Since we have now given to the Lie group a Riemannian manifold structure, we can endow $G$ with a regular Borel measure $\nu$. Notice that from the construction of the metric $\nu$ is a left-invariant measure, this also called left Haar measure. The left Haar measure is unique up to a scaling constant, determined by the choice of scalar product.
Also the scalar product in the Lie algebra induces a measure $\lebesgue$ in $\alg$\footnote{It is sufficient to consider $\alg$ with the Riemannian metric given by "copying" the scalar product at each point. This could be formalized considering $\alg$ itself a Lie group with respect to vector addition and repeating the same argument used for G}
that is invariant with respect to vector addition and unique up to a constant. 
The following Proposition gives a general formula for the change of variables in Riemannian manifolds:
\begin{prop}\label{prop:riem-change}
(Proposition 1.3 \cite{helgason2})
Let $M$ and $N$ be Riemannian  manifolds and $\Phi$ a diffeomorphism of M onto N. For $p\in M$ let $|\det(d\Phi_p)|$ denote the absolute value of the determinant of the linear isomorphism $d\Phi_p: T_pM \to T_{\Phi(p)}N$ when expressed in terms of any orthonormal bases. Then given a function $F$:
\begin{equation}
    \int_N F(q)\ d q = \int_M F(\Phi(p)) |\det(d\Phi_p)| dp
\end{equation}
if $dp$ and $dq$ denote  the  Riemannian measures on M and N, respectively
\end{prop}

In order to change variables we therefore need an orthonormal basis for the tangent space $T_aG$ at each one of the group elements $a\in G$. 

Similarly as we built the Riemannian metric, this is given by the differential of the Left group action. 

In fact given $\mathbb{B} = (\e_i)_{i\in[n]}$ a basis of the Lie algebra, then a basis $\mathbb{B}_a$ for $T_a G$ is given by $\left(\dd(L_a)_e(e_i)\right)_{i\in[n]}$. If $(\e_i)_{i\in[n]}$ is orthonormal then $\left((dL_a)_e(e_i)\right)_{i\in[n]}$ is an orthonormal basis for $T_a G$ considering $G$ endowed with the Riemannian metric defined in Equation \ref{eq:left-metric}:

\begin{equation}
    \langle \dd(L_{a})_e \e_i, \dd(L_{a})_e \e_j \rangle_a 
    = \langle \e_i,\e_j \rangle = \delta_{ij}
    \quad \forall a\in G, \quad i,j \in [n]
\end{equation}

Then with respect of this basis the matrix representation $\text{U}$ of the differential of the exponential $\dd\exp_\x$ has entries:
$$U_{ij} = \langle (\dd (L_{\exp(\x)})_e)(\e_i), \dd\exp_\x (\e_j)\rangle_{\exp(\x)} = 
\langle \e_i,  d(L_{\exp(\x)^{-1}})_{\exp(\x)}\circ d\exp_\x (\e_j)\rangle $$
Where the equality follows from \eqref{eq:left-metric}
\footnote{Notice that here in the following derivations we identify the tangent space at a point $\x$ of the Lie algebra with the Lie algebra itself.}
. From this equality it is clear that $U$ is equal to the matrix representation of the endomorphism $\dd(L_{\exp(\x)^{-1}})_{\exp(\x)}\circ \dd\exp_\x : \alg \to \alg$ with repect to the basis $\mathbb{B}$. Since the determinant an endomorphism is a quantity defined independently of the choice of the basis. The volume change term is independent on the choice of scalar product and metric and it is given by the determinant of the endomorphism $\dd(L_{\exp(\x)^{-1}})_{\exp(\x)}\circ \dd\exp_\x$ that can be computed with respect of any basis of $\alg$. 
\footnote{Notice that even if the formal construction uses an explicit choice of scalar product and basis the induced measures $\nu$ and $\lebesgue$ are independent of this choice up to a scalar multiplicative constant. Moreover since the choice of the constant for $\lebesgue$ automatically the constant for $\nu$ the change of volume term is completely independent from the choice of scalar product and basis, as showed above. Regardless of these considerations the density of the pushforward measure will in general dependent of the choice of basis and scalar product, an in depth discussion of this behaviour is given in Appendix \ref{app:choice-scalar-algebra}}

Then Theorem 1.7 of \cite{helgason1} gives a general expression of this endomorphism for every Lie group:
\begin{theorem}(Theorem 1.7 of \citet{helgason1})
Let $G$ be a Lie group with Lie algebra $\alg$. The exponential mapping of the manifold $\alg$ into $G$ has the differential:
\begin{equation}\label{eq:equality-diff}
\dd\exp_\x = \dd(L_{\exp(\x)})_e \circ \dfrac{1-\exp(-\ad_\x)}{\ad_\x},
\end{equation}
\end{theorem}
where $\dfrac{1-\exp(-\ad_\x)}{\ad_\x}$ is a formal expression to indicate the infinite power series $\sum_{k=0}^{+\infty} \frac{(-1)^k}{(k+1)!}(\ad_\x)^k$.

Now simply by composing on the left each side of \eqref{eq:equality-diff} with $d(L_{\exp(\x)^{-1}})_{\exp(\x)}$ we have that:
\begin{equation}
   \dd(L_{\exp(\x)^{-1}})_{\exp(\x)} \circ \dd\exp_\x = \dfrac{1-\exp(-\ad_\x)}{\ad_\x}:=
   \sum_{k=0}^{+\infty} \frac{(-1)^k}{(k+1)!}(\ad_\x)^k
\end{equation}

Combining this expression with Proposition \ref{prop:riem-change} we have the general expression
for the change of variables in Lie groups:

\begin{lemma}\label{lemma:lie-varchange}
Let $\lebesgue$ and $\nu$ defined as above. Let $U\subseteq \alg$ an open set in which $\exp_{|U}:U\to \exp(U)\subseteq G $ is a diffeomorphism. Let $f$ measurable function in $\alg$ and $h$ a measurable function in $G$. Then we have:
\begin{equation}\label{eq:lie-varchange}
    \int_U f\ d \lebesgue= \int_{\exp(U)} f(\exp^{-1}(a)) |J(\exp^{-1}(a))|^{-1} d \nu
\end{equation}
\begin{equation}\label{eq:lie-varchange-rev}
    \int_U h(\exp(x)) |J(\x)| d \lebesgue= \int_{\exp(U)} h\ d \nu,
\end{equation}
where:
\begin{equation}
    J(\x) = \det\left(\dfrac{1-\exp(-\ad_\x)}{\ad_\x}\right)
\end{equation}
\end{lemma}

 When we can find all eigenvalues of $\ad_\x$ the following theorem gives a closed form for $J(\x)$ . 

\begin{theorem} \label{app:computing-jacobian}
Let $G$ be a Lie Group and $\alg$ its Lie algebra, then the expression
\begin{equation} \label{eq:ap-computing-jacobian}
    J^{-1}(\x):= \det\left(\dfrac{1-\exp(-\ad_x)}{\ad_x}\right) =
    \prod_{\substack{\lambda \in \text Sp(ad_x) \\\lambda \neq 0 }} \frac{1 - e^{-\lambda}}{\lambda}, 
\end{equation} 

where $Sp(\cdot)$ is the spectrum of the operator, i.e.\ the set of its (complex) eigenvalues, i.e.\ the multiset of roots of the characteristic polynomial of the operator (in complex field), in which each element is repeated as many times as its algebraic multiplicity.
\begin{proof}
Let $P$ a matrix representation on a given basis of the endomorphism $ad_x$. Then we have:
\begin{align}
    J^{-1}(x) = \det \left( \sum_{k=0}^{+\infty} \frac{(-1)^k}{(k+1)!}P^k\right) = {\textstyle\det_{\mathbb{C}}} \left( \sum_{k=0}^{+\infty} \frac{(-1)^k}{(k+1)!}P^k\right),
\end{align}
where $\det_{\mathbb{C}}(\cdot)$ is the determinant in complex field. Formally this is the determinant applied to the complexification of the endomorphism. Now let $Q\in\GL{n, \mathbb{C}}$ such that $P= Q^{-1}(D+N)Q$ where $(D+N)$ is the Jordan normal form of $P$ where $D$ is the diagonal matrix that has as entries elements of the spectrum of $P$ and $N$ is a nilpotent matrix.
Then we have:
\begin{align}
{\textstyle\det_{\mathbb{C}}} \left( \sum_{k=0}^{+\infty} \frac{(-1)^k}{(k+1)!}P^k\right) = 
{\textstyle\det_{\mathbb{C}}} \left( \sum_{k=0}^{+\infty} \frac{(-1)^k}{(k+1)!}(D+N)^k\right) = 
{\textstyle\det_{\mathbb{C}}} \left( \sum_{k=0}^{+\infty} \frac{(-1)^k}{(k+1)!}(D)^k\right),
\end{align}
where the last equality follows from the fact that $(D+N)^k = D^k + N'$ where $N'$ is an another nilpotent matrix, and from the fact that the determinant of a triangular matrix depends only on the diagonal entries. 
Using the definition of $D$ we can then write: 
\begin{align}
 J^{-1}(\x) = {\textstyle\det_{\mathbb{C}}} \left( \sum_{k=0}^{+\infty} \frac{(-1)^k}{(k+1)!}(D)^k\right) = 
\prod_{\lambda\in Sp(\ad_\x)} \left( \sum_{k=0}^{+\infty} \frac{(-1)^k}{(k+1)!}(\lambda)^k\right)
\end{align}
Now if $\lambda = 0$ then $ \sum_{k=0}^{+\infty} \dfrac{(-1)^k}{(k+1)!}(\lambda)^k = 1$.
Else, if $\lambda \neq 0$ then $\sum_{k=0}^{+\infty} \dfrac{(-1)^k}{(k+1)!}(\lambda)^k = \dfrac{1 - e^{-\lambda}}{\lambda}$
\end{proof}
\end{theorem}

\subsection{Matrix Lie Groups} \label{app:subsec-matrix-lie-groups}
In the case of a matrix Lie group $G<\GL{n,\mathbb{R}}\subseteq \M{n,\R}$ we can exploit the fact our group is embedded in $\M{n,\R}\simeq \R^{n\times n}$ to give an alternative way to compute a matrix representation of $\dd(L_{\exp(\x)^{-1}})_{\exp(\x)}\circ \dd \exp_\x$.
This corresponds to what in the literature is known as the Left Jacobian $J_l$

Here we show how we can derive the expression of $J_l$ from the formal framework described in the previous Sections, using the additional information given by the fact that we are in a matrix Lie group. 
This is done using the fact that at each point $a\in G$ the tangent space $T_a G < T_a(\GL{n, \mathbb{R}}) \tilde= \M{n, \mathbb{R}}$ can be identified with a subspace of the real $n\times n$ matrices. 

In fact let $p\in \GL{n, \mathbb{R}}$, considering $\GL{n, \mathbb{R}}$ as an open subset of $\M{n,\mathbb{R}}$ then the canonical basis $(E_{ij})_{ij}$ of $\M{n, \mathbb{R}}$ induces the isomorphism $\psi_p: \M{n, \mathbb{R}} \to T_p(\GL{n, \mathbb{R}}),\ E_{ij} \mapsto \partial_{E_{ij}}|_p$. 
With this identification the diffe
rential of the $\exp$ is a map from $\M{n, \mathbb{R}}$ to $\M{n, \mathbb{R}}$ and can be directly computed taking derivatives. 
The same holds for the differential of the left group action. Moreover the following Lemma shows that it corresponds to a matrix left multiplication. 
With this isomorphism we can see that the differential of left multiplication corresponds exactly to left matrix multiplication: 
\begin{lemma}\label{lemma:left-action-matrix}
Let $P,Q\in \GL{n, \mathbb{R}}$ and let $L_P$ the left action of $P$ then $d( L_P)_Q$ identifying both the tangent spaces with $\M{n, \mathbb{R}}$ using the isomorphisms $\psi_P, \psi_{PQ}$ is the following function:
\begin{align}
   \dd( L_P)_Q :  &\M{n,\mathbb{R}}\to \M{n,\mathbb{R}}\\
                & X\mapsto PX 
\end{align}
\end{lemma}
\begin{proof}
Let $X\in\M{n, \mathbb{R}}$ then $\forall f \in C^{\infty}(\GL{n, \mathbb{R}})$
\begin{align}
\left[d( L_P)_Q\left({\partial_X}_{|_Q}\right)  \right] (f) = 
{\partial_X}_{|_Q} \left(f \circ L_P \right) = 
{\frac{d}{dt}}_{|_{t = 0}}\left(f \circ L_P(Q + tX)\right) = \\
= {\frac{d}{dt}}_{|_{t = 0}}f\left(L_P(Q + tX)\right) = 
{\frac{d}{dt}}_{|_{t = 0}}f\left(PQ + tPX)\right) = 
\left({\partial_{PX}}_{|_{PQ}}\right)  (f) 
\end{align}
\end{proof}

These considerations lead to the following result:

\begin{theorem}
Now let $G<\GL{n, \mathbb{R}}$ be a matrix Lie group, $\mathbb{B} := (\vv_i)_i$ a basis of the Lie algebra. Then the Lie algebra endomorphism $d(L_{\exp(X)^{-1}})_{\exp(X)}\circ d\exp_X$ has matrix representation with respect to $\mathbb{B}$:
\begin{equation}\label{left-jacobian}
    J_l(X) = \begin{bmatrix}
    \left(\exp(X)^{-1}\dfrac{\partial \exp}{\partial \vv_1}(X)\right)^{\lor}\ \vline & 
    \cdots \quad \vline&
    \left(\exp(X)^{-1}\dfrac{\partial \exp}{\partial \vv_n}(X)\right)^{\lor}
    \end{bmatrix}\in \M{n, \mathbb{R}}
\end{equation}
Which is called the left-Jacobian. Where $(\cdot)^{\lor}:= \varphi_{\mathbb{B}} : \alg \to \mathbb{R}^n$ is the ismomorphism given by the basis $\mathbb{B}$. 
\begin{proof}
Considering $G$ as embedded in $\GL{n, \mathbb{R}}$ Then the tangent space at each point can be identified with a vector subspace of $\M{n, \mathbb{R}}$. 

Then given this identification, taking $X\in \alg\subseteq\M{n, \mathbb{R}}$ the quantities $d\exp_X(\vv_i) = \frac{\partial \exp}{\partial \vv_i}(X)
\in \M{n, \mathbb{R}}$ are real valued matrices and can be simply obtained deriving the expression of the exponential in each entry. 
Moreover we have $ \left[d(L_{\exp(X)^{-1}})_{\exp(X)}\circ d\exp_X\right] (\vv_i) = \exp(X)^{-1}\frac{\partial \exp}{\partial \vv_i}(X) \in \alg\subseteq\M{n, \mathbb{R}}$  
where the equality is given by considering the left group action as the restriction of $L_{exp(X)}:\GL{n, \mathbb{R}} \to \GL{n, \mathbb{R}} $ to $\alg$ and applying the Lemma \ref{lemma:left-action-matrix}.This gives an explicit description on how the endomorphism acts on each vector of the basis. From this we can build its matrix representation $J_l(X)$. This gives us the thesis. 
\end{proof}
\end{theorem}

\section{PUSHFORWARD DENSITY} \label{app:main-theorem}
\subsection{Preliminary Lemmata}
\begin{lemma}[See \citep{duistkolk2000lie} Cor.\ 1.5.4]
For a Lie Group $G$ with algebra $\alg$ and exponential map $\exp : \alg \to G$, the set of singular points $\Sigma$ is the set:

\begin{align*}
\Sigma &=\{\x \in \alg | \det(T_\x \exp )=0\} \\
  &= \bigcup_{k \in \mathbb{Z} \setminus \{0\}}k \Sigma_1, 
\end{align*}
with
\begin{align*}
\Sigma_1 &:= \{\x \in \alg | \det \left( (\ad_\x)_\mathbb{C} - 2 \pi i I \right)=0 \},
\end{align*}
where $(\ad X)_\mathbb{C}$ denotes the adjoint representation of the real Lie algebra $\alg$ as a linear operator on the complex vector space.%\footnote{Duistermaat and Kolk, Corollary 1.5.4}
\end{lemma}

\begin{lemma} 
Let $f \in \mathbb{C}[X_1,\dots,X_n]$ be a complex polynomial viewed as a function on the real vector space $\mathbb{R}^n$:
$$ f:\, \mathbb{R}^n \to \mathbb{C},\quad x \mapsto f(x)$$
Then either $f$ is identically zero or the set of roots
$\{ x \in \mathbb{R}^n \, |\, f(x) =0\}$
has Lebesgue measure zero in $\mathbb{R}^n$.
\begin{proof}
The problem is reduced to the real polynomial $g \in \mathbb{R}[X_1,\dots,X_n]$ defined by
$$ g:=\mathrm{Re}(f)^2+\mathrm{Im}(f)^2$$
It has the same set of (real) roots as $f$ and
$g$ is identically zero if and only if $f$ is.
The statement then follows from the theorem of Okamoto. A simple proof can be found in \citep{zeroset2005}.
\end{proof}
\end{lemma}

\begin{lemma} \label{lemma:singular}
For a Lie Group $G$ with algebra $\alg$ and exponential map $\exp : \alg \to G$, the set of singular points $\Sigma$ is closed and has Lebesgue measure 0.
\end{lemma}
\begin{proof}
$\Sigma$ is closed because it is the preimage of the closed set $\{0\} \subset \mathbb{R}$ of the continuous function  $\det(T_X \exp )$.

Let $f(X)=\det \left( (\ad X)_\mathbb{C} - 2 \pi i I \right)$. $f$ is a polynomial in $X$, because $\ad$ is linear and $\det$ polynomial. $f$ can not be identically zero, as $\{0\} \not \in \Sigma$, because $\exp$ is a diffeomorphism in a neighbourhood of $0 \in \alg$ (see \citep{duistkolk2000lie} 1.3.4). Thus, the set of roots of $f$, namely $\Sigma_1$, has Lebesgue measure zero.
It follows that also $\Sigma$ has Lebesgue measure zero.
\end{proof}

\begin{defn}(Sets of Lebesgue measure 0 on a Manifold)
If M is a smooth n-manifold  we say that a subset $A \subseteq M$ has measure zero in $M$
if for every smooth chart $(U, \varphi)$ the subset $\varphi(A\cap U)\subseteq \R^n$ has n-dimensional measure zero. 
\end{defn}

\begin{lemma} \label{lemma:zero-boundary}
Let $M$ a smooth manifold. Then $\forall p \in M$ and $U$ open neighbourhood of $p$ there exists $U' \subseteq U$ open neighbourhood of $p$ such that $\partial U'$  has Lebesgue measure $0$.
\begin{proof}
Take a smooth chart $(V,\varphi)$ such that $p\in V$. Let $V':= V\cap U$ open set. 
Then $\varphi(V')$ is an open set in $\R^n$ such that $\varphi(p)\in \varphi(V')$. Take then an open ball $B(\varphi(p),r)$ with $r>0$ such that $\subseteq \varphi(V')$. If define $U':=\varphi^{-1}(B(\varphi(p),r))$ we have that $U'$ is an open neighborhood of $p$ and that $\varphi(\partial U') = \partial B(\varphi(p),r)$ has measure $0$ in $\R^n$. Then Lemma 6.6 of \cite{LeeSmooth} implies that $\partial U'$ has measure 0 in $M$
\end{proof}
\end{lemma}

\begin{lemma} \label{lemma:local-diffeomorphism}
Let $N$ and $M$ smooth manifolds of the same dimension and $F: M \to N$ a smooth map. Let $D := \{p\in M : F \text{ is a local diffeomorphism at } p \} \subseteq M$. Then $D$ can be partitioned in $D = B\bigcup \left(\cup_{k = 1}^{+\infty}A_k \right) $ such that $B$ has Lebesgue measure 0 and for every $k$ $A_k$ is an open set such that $F|_{A_k}: A_k \to F(A_k)$ is a diffeomorphism.
\end{lemma}
\begin{proof}
We first show that $D $ is open: $\forall p\in D$ since $F$ is a local diffeomorphism at $p$ there exists a neighbourhood $U_p\ni p$ such that $F|_{U_p}$ is a diffeomorphism. Then $U_p \subseteq D$. This shows that $\mathring{D} = D$ thus $D$ is open. Therefore $D$ inherits a manifold structure from $M$ as a sub-manifold, meaning that $D$ is second countable, implying $D$ is Lindel\"of (see \citep{lee2010introduction}, Thm.\ 2.50). This means that every open cover has a countable subcover. 
\par
For every $p\in D$ consider $U_p\in D$, neighbourhood of $p$ such that $F|_{U_p}$ is a diffeomorphism. Then by Lemma \ref{lemma:zero-boundary} there exists $U'_p \subseteq U_p$ open neighbourhood of $p$ such that $\partial U'_p$  has Lebesgue measure $0$. 
Consider then the cover $\{U'_p : p\in D\}$, by Lindel\"of 
property it has a countable subcover $\{A'_n\}_{n=1}^{+\infty}$. We then iteratively build the sets $A_1:= A'_1$ , $A_n:= A'_n\setminus \overline{\left({\cup_{k=1}^{n-1}A'_k}\right)}$. Then by construction the sets $A_n$ are open and $F|_{A_n}$ is a diffeomorphism. Moreover defining $B:= D\setminus \left({\cup_{k=1}^{+\infty}A_k}\right)$ we are are left to show that $B$ has Lebesgue measure $0$. This simply follows from the fact that by construction $B\subseteq \cup_{k=1}^{+\infty}\partial A'_k$ and that the sets $A'_k$ were defined to have boundary of Lebesgue measure $0$. 
To see that $B\subseteq \cup_{k=1}^{+\infty}\partial A'_k$ consider $b\in B$ and define the set $N_b =\{n\in {\mathbb{N}^+} : b\in A'_n\}$, the set is not empty since the sets $A'_k$ form a cover of $D$. Let then $m\in N_b$ be the smallest element in $N_b$. Since $b\in B$ then $b\not\in A_1 = A'_1$ therefore $m>1$. Moreover $b\not\in A_m$ and since $b\in A'_m$ we have that $ b\in  \overline{\left({\cup_{k=1}^{m-1}A'_k}\right)} = \left({\cup_{k=1}^{m-1}A'_k}\right) \cup \partial \left({\cup_{k=1}^{m-1}A'_k}\right)$. By definition of $m$, $b \not \in \left({\cup_{k=1}^{m-1}A'_k}\right)$, then $b \in \partial \left({\cup_{k=1}^{m-1}A'_k}\right)\subseteq  {\cup_{k=1}^{m-1}\partial A'_k} \subset {\cup_{k=1}^{+\infty}\partial A'_k}$
\end{proof}

\subsection{Main Theorem}

Now suppose we have samples from a measure $\measurealg \ll \lebesgue$ with density $r$. We can then "push" the samples to elements in $G$ through the $\exp$ map. The resulting samples will be distributed according to the  pushforward measure $\exp_*(\measurealg)$ on $G$. 
The following theorem ensures that $\exp_*(\measurealg)$ is a.c. with respect to the left Haar 
measure $\nu$ and gives an expression for the density
\begin{theorem} \label{thm:main}
Let $G$, $\alg$, $\measurealg, \lebesgue, \nu$ defined as above. Then $\exp_*(\measurealg)\ll \nu$ with density:
\begin{equation}
    p(a) = \sum_{\{\x\in \alg : \exp(\x) = a\}} r(\x) |J(\x)|^{-1},
\end{equation}
where $J(\x):=  \det\left(\dfrac{1-\exp(-\ad_\x)}{\ad_\x}\right) = \det \left(\sum_{k=0}^\infty \dfrac{(-1)^k}{(k+1)!} (\ad_\x)^k \right)$
\end{theorem}

\begin{proof}
Using Lemma \ref{lemma:singular}, we partition $\alg$ in the open set $A$ such that $A$ is the set of points in which $\exp$ is a local diffeomorphism and $\Sigma:=\alg\setminus A$.

Using Lemma \ref{lemma:local-diffeomorphism}, we further partition $A$ in countably many open sets $\{A_k\}_{k \in I}$, for some index set $I$, and a set $B$, such that $\exp|_{A_k}$ is a diffeomorphism for all $k$ and $B$ is of Lebesgue measure 0. Define for $S \in \mathcal{B}[\alg]$:
$$\measurealg_{S} : \mathcal{B}[\alg] \to [0, 1] : E \mapsto \measurealg(E \cap S)$$
Then we have, since $\measurealg \ll \lebesgue$ and $\lebesgue(\Sigma)=\lebesgue(B)=0$:
$$\measurealg = \measurealg_\Sigma + \measurealg_B + \sum_{k\in I}\measurealg_{A_k} = \sum_{k\in I}\measurealg_{A_k}$$
Consider the pushforward measure $\exp_*(\measurealg)$, we have for all $D\in \mathcal{B}[G]$:
\begin{align*}
(\exp_*(\measurealg))(D) &= \measurealg(\exp^{-1}(D))
\\ &= \sum_{k\in I}\measurealg_{A_k}(\exp^{-1}(D)) \\
&= \sum_{k\in I}(\exp_{*}(\measurealg_{A_k}))(D) \\
&= \sum_{k\in I}((\exp_{|A_k})_*(\measurealg))(D) \\
\implies \exp_*(\measurealg) &= \sum_{k\in I}(\exp_{|A_k})_*(\measurealg),
\end{align*}
where we define:
$$((\exp_{|A_k})_*(\measurealg))(D)=\measurealg(\exp^{-1}(D) \cap A_k)$$
Notice that $\exp_{|A_k}:A_k \to \exp(A_k)$ is now a diffeomorphism, so the change of variable formula in (\ref{eq:lie-varchange}) can be applied:
\begin{align}
((\exp_{|A_k})_*(\measurealg))(D)
&= \int_{\exp_{|A_k}^{-1}(D)\cap A_k} r \ d\lebesgue \nonumber \\
&= \int_{D\cap \exp(A_k)}(r\circ\exp_{|A_k}^{-1})\cdot (|J|^{-1}\circ{\exp_{|A_k}^{-1}}) \ d\nu
\end{align}
Then $(\exp_{|A_k})_*(\measurealg)\ll \nu$ and since $\exp_*(\measurealg)= \sum_{k\in I}((\exp_{|A_k})_*(\measurealg))$ then $\exp_*(\measurealg)\ll \nu$. 

In order to find the expression for the density we observe that $(\exp_{|A_k})_*(\measurealg)$ has density $r(\exp_{|A_k}^{-1}(a))J^{-1}(\exp_{|A_k}^{-1}(a))\charf_{\exp(A_k)}(a)$ where $a\in G$ and $\charf$ is the indicator function. 
Then we have that the density of $\exp_*(\measurealg)$ with respect to $\nu$ is 
\begin{align}
\sum_{k\in I}r(\exp_{|A_k}^{-1}(a))|J^{-1}(\exp_{|A_k}^{-1}(a))|\charf_{\exp(A_k)}(a) =   
\sum_{\{\x\in \alg : \exp(\x) = a\}} r(\x) |J(\x)|^{-1},
\end{align}
where the last equality is true almost everywhere in $G$. This can be seen if we define the set $N\subseteq G$ as all the points $p\in G$ in which $\{\exp_{|A_k}^{-1}(p) : k \in I\} \neq \{x\in \alg : \exp(x) = p\}$. Then $N$ has Lebesgue measure 0. 
In fact $N \subseteq \exp\left(B\cup \Sigma \right)$ and since $B\cup \Sigma$ has measure zero in $\alg$ and $\exp$ is smooth then by Theorem 6.9 in \cite{LeeSmooth} $\exp\left(B\cup \Sigma \right)$ has measure $0$.

\qedhere
\end{proof}

\section{COMPUTATIONAL COMPLEXITY} \label{app:computational-complexity}

\subsection{Complexity of the Reparameterization Trick}
In this appendix we will analyze the complexity of performing the reparameterization trick when working with a Lie group $G$ of dimension $n$.
For simplicity we will assume in the following considerations that $G$ is a matrix Lie group. The complexity is given by the cost of computing the $\exp$ map and its differential. The $\exp$ map for a matrix lie group is given by the matrix exponential
\begin{equation}
 \exp\left(X\right) = \sum_{k=0}^\infty \frac{X^k}{k!} \quad X \in \M{n, \mathbb{R}},
\end{equation}
which involves an infinite summation. In general the worst case complexity for computing a good approximation of the above expression is $O(n^3)$ \footnote{The interested reader is referred to \citep{moler2003nineteen} for a survey on the possible ways to compute matrix exponentials with a detailed explanation for the complexity of each method.}. 

For the differential of the $\exp$ map, the computation via the left-Jacobian \eqref{left-jacobian} is generally also cubic in $n$, as it involves a matrix inversion. The alternative is to use equation \eqref{eq:ap-computing-jacobian} in Theorem \ref{app:computing-jacobian}, in which case the complexity is cubic in $n$ as well. In fact because the Lie algebra $\mathfrak{g}$ is a vector space of dimension $n$, then since $\ad_\x\in \text{End}(\mathfrak{g})$ fixed a basis for $\mathfrak{g}$, $\ad_\x$ has a matrix representation as an element of $\GL{n, \mathbb{R}}$. One can either compute the $\exp$ of this matrix, or find its eigenvalues, both operations are cubic in $n$. 

Despite the above considerations, for specific Lie Group there might exist specific analytic calculations to derive closed form expressions for the exponential map and for the eigenvalues of the adjoint, using group specific properties. This can in practice lead to a significant reduction in computational complexity, as it is shown in the specific examples of Section \ref{sec:examples}. 

\subsection{Approximation of Infinite Summations}
In Appendix \ref{app:main-theorem} we have proven that the the pushforward measure of a probability measure in the Lie algebra is a well defined measure on the Lie Group, with a density with respect to the Haar measure on the group. However the expression of the density at a point depends on a potentially infinite summation.
In general since $ 1 = [\exp_*(\measurealg)](G) = \int_{G}p d\nu$, the density is finite almost everywhere in $G$. This means that the infinite series can be truncated at the $N$-th term, still retaining an arbitrarily good approximation (that depends on $N$). In practice we have observed that when using an exponentially decaying distribution on the lie algebra, only an handful of terms are sufficient to get a good approximation. 
However it is difficult to derive general bounds and to determine a priori a good value for $N$, as this will greatly depend on the choice of base distribution, on the specific Lie Group and on the way we decide to enumerate the points in $\exp^{-1}(a)$. 

A possible alternative to avoid infinite summations is to use a compactly supported distribution, this reduces the infinite series to a finite summation, since the terms become definitely $0$. 
Notice that since compactly supported functions are dense in $L^1(\R^n)$ and that $r\in L^1(\R^n)$, there is always a compactly  supported function that approximates $r$ arbitrarily well. 

Connected to this approach, it is possible to choose a density supported in the injectivity radius of the exponential map. The summation then reduces to one term. Moreover if the base density $r$ is smooth then the density of the pushforward will also be smooth. 

\section{CHOICE OF BASIS AND SCALAR PRODUCT IN THE LIE ALGEBRA} \label{app:choice-scalar-algebra}
In the previous Sections the starting point for obtaining a reparameterizable density  on the Lie Group $G$ was using a reparameterizable density on the corresponding Lie algebra $\alg$. 

Since computations usually can only be done on real values we need a concrete representation of the abstractly defined Lie algebra $\alg$ as some real vector space $\mathbb{R}^n$.

This amounts to say that we need to choose a concrete basis $b=(b_1,\dots,b_n)$ with the $b_i \in \alg$ and identify linear combinations $v=\sum_{i=1}^n x_i \cdot b_i \in \alg$
with the corresponding vector $x=(x_1,\dots,x_n)^T \in \mathbb{R}^n$. Every such a choice of basis gives us a linear isomorphism:
$$ \psi_b:\;\mathbb{R}^n \cong \alg, \qquad x=(x_1,\dots,x_n)^T \mapsto \sum_{i=1}^n x_i \cdot b_i$$
Furthermore, the standard scalar product on $\mathbb{R}^n$ induces a scalar product $<.,.>$ on $\alg$ via the above isomorphism.

We can then proceed in two ways: 
\begin{enumerate}
    \item In case we can directly and intrinsically define a (probability) measure $\measurealg$ on $\alg$
    then we can take any basis $b$ and push $\measurealg$ via $\varphi_b:=\psi_b^{-1}$ to $\mathbb{R}^n$ (to get $\varphi_{b,*}\measurealg$).
    We then can use the real valued representation there to reparameterize the corresponding density. All results can then be pulled back to $\alg$ with $\psi_b$.

    \item The second way is to start directly  from a reparameterizable measure $\measurealg'$ %density $r'(x;\theta)$ 
    on $\mathbb{R}^n$ and then define: $\measurealg_b:=\psi_{b,*}\measurealg'.$ 
\end{enumerate}

Even though both view points seem to be equivalent, only the first one is independent of the representation as the ``true'' measure $\measurealg$ on $\alg$ was already given. The second method will highly depend on the choice of the basis $b$ and the measure $\measurealg'$. Therefore, if possible, the first approach is preferred. However in practice specifying measures or densities directly in $\mathbb{R}^n$ is easier as the abstract definition of $\alg$ is not directly accessible. We will discuss this further in the following.

As mentioned before, the standard scalar product on $\mathbb{R}^n$ induces a scalar product $<.,.>$ on $\alg$ via aboves isomorphism and thus a left-invariant Riemannian metric on $G$. 

So the whole Riemannian geometric structure of the Lie group $G$ is sensitive to the choice of the basis on $\alg$. Also, if we would now sample from a skewed distribution $p(x)$ on $\mathbb{R}^n$ and push the samples to $G$ via the maps:
\begin{align}
    \mathbb{R}^n \stackrel{\psi_b}{\cong} \alg \stackrel{\exp}{\longrightarrow} G, \nonumber
\end{align}
then these would in general not be the same as when using another basis for the isomorphism. To summarize, we need to choose the basis carefully and keep the dependence on it in mind.

Now let us assume that we already have a specified scalar product $<.,.>$ on $\alg$.  Then a natural choice would be to take a orthonormal basis
$b=(b_1,\dots,b_n)$ w.r.t.\ the given scalar product, i.e.\ we have:
$ <b_i,b_j>= \delta_{i,j}.$
Then still skewed distributions $p(x)$ on $\mathbb{R}^n$ would be mapped to different distributions under a different choice of orthonormal basis. In case $p(x)$ is invariant under orthonormal transformations (i.e.\ $p(g.x)=p(x)$ for all $g \in O(n)$ and $x \in \mathbb{R}^n$) like Normal distributions of form $p(x)=\mathcal{N}(x|0,\sigma^2 \cdot I)$ 
then the pushforward of $p(x)$ onto $G$ would not depend on the choice of orthonormal basis.

But note that the notion of orthonormality strongly depends on the chosen scalar product $<.,.>$ on $\alg$ and the number of choices one can make are i.g.\ infinite. Different scalar products lead to different orthonormal basises.

So it is left to discuss how to choose a scalar product on $\alg$ or a Riemannian metric $g$ on $G$, resp.. To reduce the number of Riemannian metrics we can impose additional desirable properties onto them, like bi-invariance.\\

\begin{theorem}[See \citep{milnor1976,ab-lie2015}]
\begin{enumerate}
\item Any Lie group $G$ that is isomorphic to the direct product of a compact Lie group $K$ and $\mathbb{R}^n$,  $n \ge 0$, admits a bi-invariant (i.e.\ left- and right-invariant) Riemannian metric $g$.
\item If $G$ is connected then also the reverse statement holds.
\item If $G$ admits a bi-invariant Riemannian metric then the Lie exponential map and the Riemannian exponential map at the identity agree.
\item If $G$ is a compact and simple Lie group then the bi-invariant Riemannian metric is unique up to a positive constant $c >0$.
\end{enumerate}
\end{theorem}

 It turns out that for certain types of Lie groups there is even a natural choice of scalar product, the so called \emph{negative Killing form}.\\

\begin{theorem}[See \citep{milnor1976,ab-lie2015}]
Let $G$ be a Lie group and for $x,y \in \alg$ define the \emph{negative Killing form} as: 
$$ <x,y> := -  \trace\left( \ad_x \circ \ad_y \right)$$
We then have the following results:
\begin{enumerate}
    \item $G$ is semisimple iff and $<.,.>$ is non-degenerate.
    \item If $G$ is semisimple and compact then $<.,.>$ induces a bi-invariant Riemannian metric on $G$.
\end{enumerate}
\end{theorem}

\subsection{Summary}

\begin{enumerate}
    \item Consider the case that we have a simple and compact Lie groups $G$ (e.g.\ $SO(2)$ or $SO(3)$).
    \item Then take the negative Killing form (up to scale $c>0$) as scalar product on $\alg$:
    $$ <x,y> := -  \trace\left( \ad_x \circ \ad_y \right).$$
    \item The left multiplications $L_a$ of $<.,.>$, for $a \in G$, then induces a bi-invariant Riemannian metric $g$ on $G$.
    \item $g$ induces the bi-invariant Haar measure $\measurealg_g$ on $G$, which on arbitrary local charts is given by the density $\sqrt{|\det(g)|}$ w.r.t.\ Lebesgue measure.
    \item In case we can compute $\measurealg_g(G)$, re-scaling the scalar product by multiplying it with the factor $c:= \frac{1}{\sqrt[n]{\measurealg_g(G)^2}}$ with $n=\dim(G)$ makes the then induced bi-invariant Haar measure normalized (i.e.\ $\measurealg_g(G)=1$).
    \item In any case, choose a orthonormal basis $b_1,\dots,b_n$ of $\alg$ w.r.t.\ $<.,.>$ and fix the isomorphism:
    $$ \varphi_b:\;\mathbb{R}^n \cong \alg, \qquad x=(x_1,\dots,x_n)^T \mapsto \sum_{i=1}^n x_i \cdot b_i$$
    \item Then the pushforward (via $\exp$) onto $G$ of probability distributions $p(x)$ on $\mathbb{R}^n$  that are invariant under $O(n)$ w.r.t.\ $<.,.>$  are independent of the chosen basis and independent of the chosen bi-invariant metric up to scale. 
    \item For example for the Normal distribution $p(x)=\mathcal{N}(x|0,\sigma^2\cdot I)$ this basically just reduces to the choice of variance $\sigma^2$ (even when not normalized, since multiplication with $c>0$ only changes the variance).
\end{enumerate}

\begin{rem}
If $G$ is only a semisimple Lie group then the negative Killing form $<.,.>$ on $\alg$ can still be used to induce a bi-invariant pseudo-Riemannian metric on $G$ and thus a bi-invariant Haar measure $\measurealg_g$, which still on local oriented coordinates is given by the density $\sqrt{|\det(g)|}$ w.r.t.\ the Lebesgue measure.
\end{rem}

\end{document}